\pdfoutput=1
\documentclass{article}

\PassOptionsToPackage{numbers, compress, sort}{natbib}

\usepackage[main, final]{neurips_2026}

\usepackage[utf8]{inputenc}
\usepackage[T1]{fontenc}
\usepackage{hyperref}
\usepackage{url}
\usepackage{booktabs}
\usepackage{amsfonts}
\usepackage{nicefrac}
\usepackage{microtype}
\usepackage[table]{xcolor}

\usepackage{amsmath, amssymb, amsthm, verbatim, mathtools}
\usepackage{graphicx}
\usepackage{appendix}
\usepackage{subcaption}
\usepackage{enumitem}
\usepackage{wrapfig}
\usepackage{algorithm}
\usepackage{algpseudocode}
\usepackage{multirow}
\usepackage{tabularx}

\DeclareMathOperator*{\Fix}{\operatorname{Fix}}
\newtheorem{assumption}{Assumption}
\newtheorem{theorem}{Theorem}
\newtheorem*{theorem*}{Theorem}
\newtheorem{definition}{Definition}
\newtheorem*{corollary*}{Corollary}
\newtheorem*{remark}{Remark}

\title{Root-Selecting Fixed-Point Inversion for Rectified Flows via Trajectory Straightness}

\author{%
    \textbf{Semin Kim ~~~~\ Jihwan Yoon ~~~~\ Seunghoon Hong} \\
    KAIST \\
    \small{
    \texttt{\{seminkim, jihwan.yoon, seunghoon.hong\}@kaist.ac.kr}
    }
}

\begin{document}

\maketitle

\begin{abstract}

Finding the initial noise that generates a given data sample, known as inversion, is a key component for downstream applications such as training-free image editing.
Existing fixed-point inversion methods improve inversion accuracy by formulating each inversion step as a fixed-point problem, but they lack a principled mechanism for selecting among multiple fixed-point solutions that can arise in practice.
We observe that different selections induce different inversion trajectories, leading to substantial variation in reconstruction and editing quality.
For rectified flows, we further find that this variation is closely associated with trajectory straightness, motivating straightness as a principled selection criterion.
We propose SelFix, a fixed-point inversion method that selects fixed-point solutions inducing straighter inverse trajectories while retaining convergence to an exact inverse root under standard local assumptions.
Experiments on FLUX.1-dev and PIE-Bench show that SelFix improves fixed-point inversion, achieving stronger real-image reconstruction and better source-preserving prompt-based editing than prior inversion baselines.
The code is available at \url{https://github.com/seminkim/selfix}.

\end{abstract}

\section{Introduction}
Recent advances in text-to-image generative models have enabled diverse, high-quality image generation, with leading backbones increasingly shifting from diffusion models~\cite{rombach2022high-ldm, ramesh2022hierarchical-dalle2, podell2024sdxl} to rectified flows~\cite{liu2022rectifiedflow, BlackForestLabs2024FLUX1, esser2024sd3, flux-2-2025}.
Accordingly, accurate inversion of rectified flows, which maps a real image back to an initial noise, has recently received increasing attention~\cite{deng2025fireflow, jiao2026unieditflow, wang2025taming-rfsolver, wang2026freelunch}.
Accurate inversion enables various applications that reuse the pretrained generative backbone, such as \emph{training-free} prompt-based image editing~\cite{mokady2023nulltext, wang2025taming-rfsolver, brack2024ledits, hertz2023prompt, pix2pix, masactrl}.

A central difficulty in inversion is discretization error.
Naively integrating the learned dynamics in the reverse direction does not exactly invert the finite-step solver, especially when the learned trajectory is curved or the solver uses only a small number of steps.
To reduce this error, previous methods often change the conditioning branch~\cite{miyake2023negative}, optimize auxiliary variables~\cite{dong2023pti,mokady2023nulltext}, or modify the sampler to couple inversion and reconstruction streams~\cite{wallace2023edict}.
Among the various approaches, fixed-point inversion methods~\cite{pan2023aidi, garibi2024renoise, samuellightning} have shown promise for more accurate inversion.
Fixed-point inversion provides an alternative, principled approach; it keeps the chosen discrete solver fixed and formulates each inversion step as a local inverse equation induced by that solver step and solves it iteratively.
When the iteration converges, the recovered point is an exact inverse of the chosen discrete forward solver, without requiring auxiliary-variable optimization or changes to the sampling procedure.

Existing fixed-point inversion methods focus on finding a single valid inverse point efficiently~\cite{pan2023aidi, samuellightning}.
Finding any single root is sufficient when the local inverse equation has a unique solution, or when all solutions lead to similar downstream behavior.
In practice, however, we observe that the learned and discretized local inverse equation can admit multiple distinct approximate fixed-point roots.
In this multi-root regime, achieving a small fixed-point residual alone is not enough.
Different candidates with comparable local residuals can lead to substantially different inversion trajectories, affecting downstream reconstruction or editing quality.
We therefore view fixed-point inversion not only as solving a local inverse equation, but also as selecting a desirable solution among multiple valid candidates, as illustrated in Fig.~\ref{fig:overview}.

\begin{figure}[t]
    \centering
    \includegraphics[width=\linewidth]{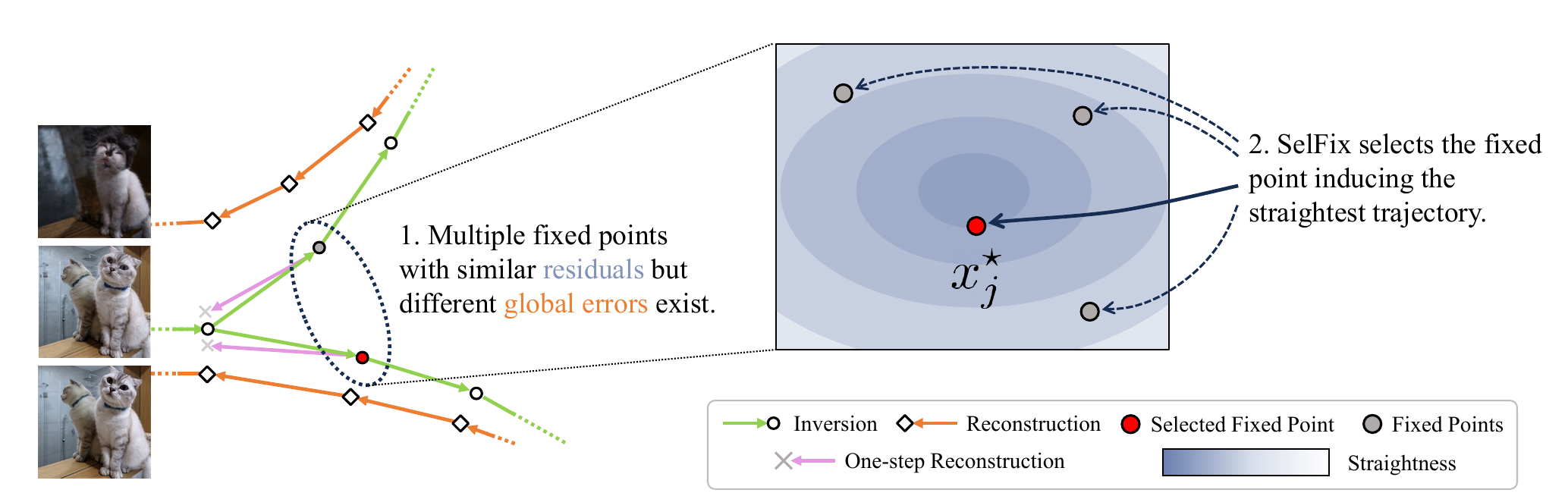}
    \caption{
    \textbf{Overview of SelFix.}
    Fixed-point inversion may produce multiple approximate fixed points with similar residuals but different global reconstruction errors.
    SelFix uses a causal straightness selector to guide fixed-point iteration toward straighter inverse trajectories, empirically improving reconstruction.
    }
    \label{fig:overview}
\end{figure}

For rectified flows in particular, a natural criterion for the selection problem is trajectory straightness.
Rectified flows are designed so that straighter trajectories are easier to simulate with coarse discrete solvers~\cite{liu2022rectifiedflow}, whereas curved trajectories make local solver approximations less reliable.
Building on our empirical observation that reconstruction error is closely associated with inversion-path straightness, we propose \textbf{SelFix}, which \textbf{sel}ects a \textbf{fix}ed point that induces a straighter inversion trajectory{}, thereby improving rectified flow inversion.
Viewing inversion as the selection of a particular fixed point, we derive a proxy for rectified flow's trajectory straightness~\cite{liu2022rectifiedflow} and use it as a criterion for selecting a root.
Then, we design an iterative algorithm based on Halpern iteration~\cite{halpern1967fixed}, which blends the original fixed-point map with a straightness anchor derived from this criterion.
We also prove that under standard assumptions, the proposed iteration converges to an exact fixed-point inverse that minimizes the straightness selector over the fixed-point set.

Our contributions are threefold.
First, we formulate fixed-point inversion as root selection over the exact local inverse set, motivated by the observation that multiple approximate roots can have similar local residuals but substantially different global reconstruction errors.
Second, we introduce a straightness-guided selector derived from rectified-flow trajectory straightness~\cite{liu2022rectifiedflow}.
Third, we propose a theoretically grounded anchored fixed-point algorithm that selects a straighter inverse root while preserving exact fixed-point inversion asymptotically.
Experiments on real-image reconstruction and prompt-based editing show that our method produces straighter inverse trajectories, yielding lower reconstruction error and better background preservation in image editing.

\section{Preliminaries}
\label{sec:prelim}

\subsection{Rectified Flow}
\label{subsec:prelim_rf}

Let $X_0\sim p_{\text{data}}$ and $X_1\sim p_{\text{noise}}$, where $p_{\text{noise}}$ is typically a standard Gaussian distribution.
Given a pair $(X_0, X_1)$, rectified flow~\cite{liu2022rectifiedflow} learns a time-dependent velocity field $u_\theta:\mathbb{R}^d\times[0,1]\to\mathbb{R}^d$ by matching the velocity of the linear interpolation
\begin{equation}
X_t=(1-t)X_0+tX_1.
\label{eq:interpolation}
\end{equation}
Since $\dot X_t = X_1-X_0$, the standard conditional flow matching objective is
\begin{equation}
    \mathcal{L}_{\mathrm{CFM}}(\theta)
    =
    \mathbb{E}_{t,X_0,X_1}
    \left[
        \left\|
        u_\theta(X_t,t) - (X_1-X_0)
        \right\|^2
    \right],
    \label{eq:cfm_objective}
\end{equation}
where $t\sim \text{Unif}[0,1]$.
The learned velocity field induces the probability-flow ODE
\begin{equation}
    \frac{d x_t}{dt}=u_\theta(x_t,t).
    \label{eq:model_ode}
\end{equation}
Data can be generated by numerically integrating Eq.~\eqref{eq:model_ode} backward in time.
Specifically, let $1=t_T>t_{T-1}>\cdots>t_0=0$ be a discrete time grid, and let $x_j$ denote the state at time $t_j$.
Given a one-step numerical solver $f_{t_j \to t_{j-1}}$, generation amounts to the recurrence $x_{j-1}=f_{t_j\to t_{j-1}}(x_j)$, starting from random noise $x_T$ and ending at the generated sample $x_0$.
For the Euler solver in particular,
\begin{equation}
    \label{eq:euler_gen}
    x_{j-1} = f_{t_j\to t_{j-1}}(x_j)  = x_j - h_ju_\theta(x_j,t_j), ~~~~h_j=t_{j}-t_{j-1}.
\end{equation}

A central property of rectified flow is that straight trajectories are easier to integrate with a small number of solver steps.
For a continuous path $Z=\{Z_t\}_{t\in[0,1]}$, the trajectory straightness is commonly measured by:
\begin{equation}
    \label{eq:rf_straightness}
    S(Z) = \int_0^1 \mathbb{E}
    \left[
    ||  \dot Z_t - (Z_1-Z_0) || ^2
    \right] dt,
\end{equation}
where $S(Z)=0$ indicates perfectly straightened paths~\cite{liu2022rectifiedflow}.
In practice, however, the learned velocity field may induce curved trajectories due both to averaging over many noise--data pairs and to model approximation error.
Generation along such curved trajectories induces discretization error for solvers such as Euler, which rely on a constant-velocity approximation over each interval.

\subsection{Inversion}
\label{subsec:prelim_inversion}

Inversion is the problem of finding initial noise $x_T$ that generates the given real data $x_0$.
This is typically done by iteratively recovering the generation trajectory, starting from the data $x_0$.
At the $j$-th iteration, the goal is to recover a state $x_j$, given the previously obtained states $x_{<j}$.
With the definition of residual at time $t_j$ as
\begin{equation}
    r_j(x)
    :=
    f_{t_j\to t_{j-1}}(x)-x_{j-1},
    \label{eq:inversion_residual}
\end{equation}
the goal for each inversion step is to find $x$ that satisfies zero residual:
\begin{equation}
    r_j(x)=0,
    \qquad\text{equivalently}\qquad
    f_{t_j\to t_{j-1}}(x)=x_{j-1}.
    \label{eq:exact_local_inversion}
\end{equation}
Solving Eq.~\eqref{eq:exact_local_inversion} is generally challenging, since it requires inverting the solver $f$ which itself depends on the learned velocity $u_\theta$.
The simplest strategy is to integrate Eq.~\eqref{eq:model_ode} in the inverse direction with the same solver.
For instance, for Euler solver, this gives $\hat x_{{j}} = x_{{j-1}} + h_j u_\theta(x_{{j-1}},t_{j-1})$.
This estimate is exact only in special cases, for example when the velocity remains constant across the steps, \emph{i.e.}, $u_\theta(x_{{j-1}},t_{j-1}) = u_\theta(x_{{j}},t_{j})$.
In general, however, this assumption does not hold due to the non-straightness in the learned velocity (Sec.~\ref{subsec:prelim_rf}), incurring errors in the inversion process.

\subsection{Fixed-Point Inversion}
\label{subsec:prelim_fpi}

Fixed-point inversion improves upon direct reverse integration by explicitly solving the local inverse equation in Eq.~\eqref{eq:exact_local_inversion}.
Define the fixed-point map
\begin{equation}
    P_j(x)
    :=
    x-r_j(x).
    \label{eq:fixed_point_map}
\end{equation}
Then being the root of the fixed-point equation $x=P_j(x)$ is equivalent to being the root of $r_j(x)=0$.
Hence, the zero set of $r_j$ is the same as the fixed-point set of $P_j$:
\begin{equation}
    S_j
    =
    \{x:r_j(x)=0\}
    =
    \Fix(P_j),
    \qquad
    \operatorname{Fix}(P_j):=\{x:P_j(x)=x\}.
    \label{eq:root_fix_equivalence}
\end{equation}
Then the inversion can be done by running a fixed-point solving algorithm on $P_j$.
One basic algorithm is applying Picard iteration:
\begin{equation}
    x_j^{k+1}=P_j(x_j^k).
    \label{eq:naive_fpi}
\end{equation}
If Eq.~\eqref{eq:naive_fpi} converges to some $x_j^\star$, then by construction $r_j(x_j^\star)=0$, so the per-step inversion is solved exactly.
Thus, if converged, fixed-point inversion gives exact inversion asymptotically, even with a coarse timestep discretization which induces large discretization error.

While empirically effective in practice, prior fixed-point inversion methods solve Eq.~\eqref{eq:root_fix_equivalence} for obtaining any single element of $S_j$.
When $\Fix(P_j)$ contains multiple roots, Eq.~\eqref{eq:naive_fpi} does not specify which exact inverse root should be preferred.
In such a case, the selected limit simply depends on initialization or the local geometry of $P_j$.
In the following section, we demonstrate that the learned model may admit multiple fixed points for Eq.~\eqref{eq:fixed_point_map}, and selecting a fixed point that induces a straighter trajectory can improve the inversion.

\section{Method}
\label{sec:method}

\subsection{Fixed-Point Inversion as Root Selection}
\label{sec:root-selection}

Existing fixed-point inversion methods seek an arbitrary element of $S_j$, which is sufficient when $S_j$ is a singleton.
When multiple roots exist, however, different choices of $x_j\in S_j$ can induce different inverse trajectories, even though they satisfy the same local reconstruction equation.
Thus, we view fixed-point inversion not only as root finding, but also as root selection.

For the local convergence analysis, we restrict attention to a region of interest $C_j\subset\mathbb{R}^d$ where the fixed-point iteration is expected to operate, and define the local exact inverse root set
\begin{equation}
    S_j^C
    :=
    S_j\cap C_j
    =
    \operatorname{Fix}(P_j)\cap C_j.
    \label{eq:local_root_set}
\end{equation}

We then formulate root-selecting inversion as
\begin{equation}
    x_j^\star
    \in
    \arg\min_{x\in S_j^C}
    \phi_j(x),
    \label{eq:root_selection_general}
\end{equation}
where $\phi_j$ is a selection criterion.
This formulation separates exactness from selection: the constraint $x\in S_j^C$ enforces exact local inversion, while $\phi_j$ specifies which exact root is preferred.

\paragraph{Empirical Existence of Multiple Fixed Points}
\label{par:empirical_multiroot}

\begin{wrapfigure}{r}{0.33\textwidth}
    \vspace{-1em}
    \centering
    \includegraphics[width=\linewidth]{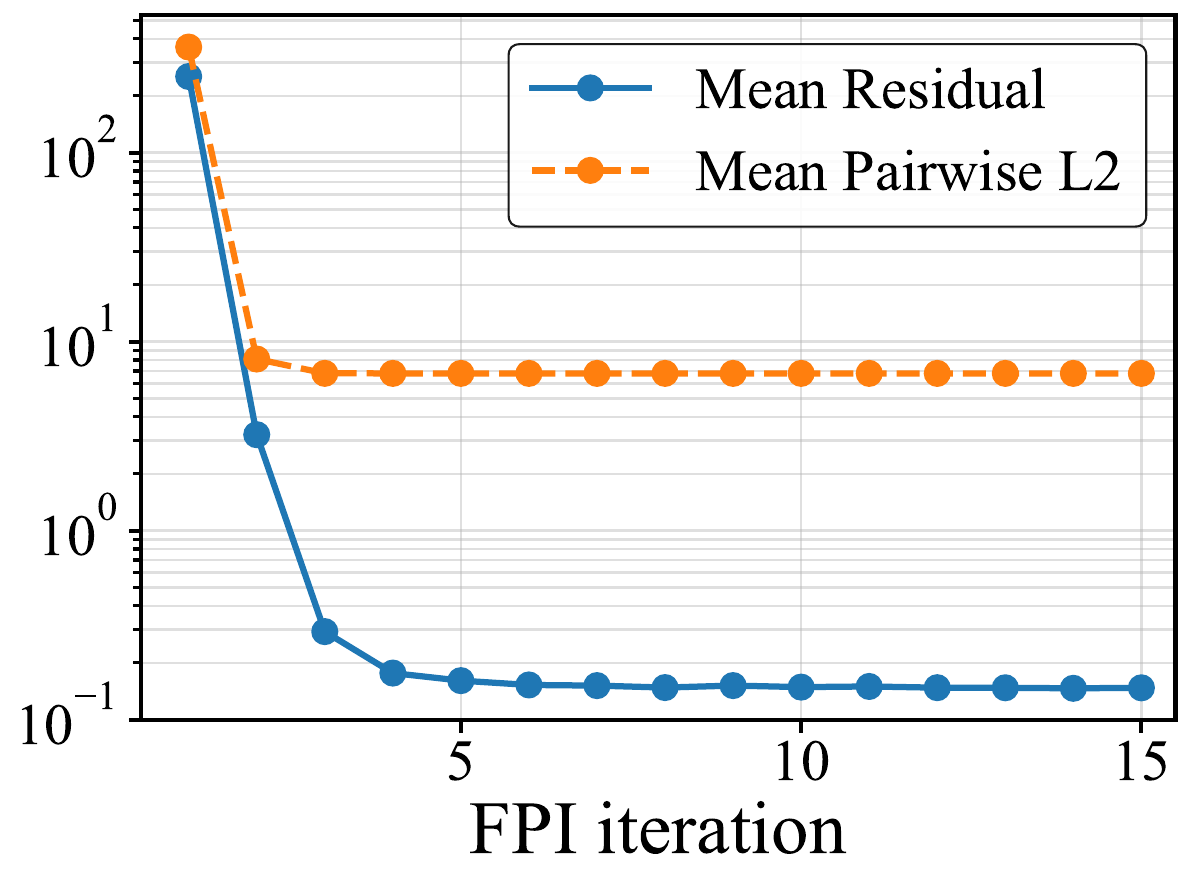}
    \caption{
    Comparison of average residual and average pairwise distance between different initializations.
    Analysis details, along with results across other backbones and datasets, are available in App.~\ref{app:multiroot-details}.
    }
    \label{fig:multiroot}
\end{wrapfigure}

To show that the learned discrete inversion equation can admit multiple approximate fixed-point roots in practice, we analyze fixed-point inversion under different initializations.
Specifically, we repeat fixed-point inversion on the same image with different random initializations, and track two quantities during the inner fixed-point iteration. The first is the mean fixed-point residual, computed as the average of $\|P_j(x)-x\|_2$ over differently initialized iterates. The second is the mean pairwise $L_2$ distance, computed between the iterates obtained from different initializations at the same timestep and inner iteration.
As shown in Fig.~\ref{fig:multiroot}, even near convergence where the mean fixed-point residual is small, differently initialized iterates tend to remain far from each other; their mean pairwise distance remains an order of magnitude larger than the mean fixed-point residual.
Because the residual is small but the pairwise distance remains large, the converged iterates behave as distinct near-roots rather than variants of a single solution.
This observation suggests that the learned discrete inversion equation can have multiple distinct approximate fixed-point roots in practice.
We emphasize, however, that this multiplicity claim concerns the learned and discretized local inversion equation in Eq.~\eqref{eq:root_fix_equivalence}; it does not claim non-injectivity of the ideal continuous-time ODE flow.

\subsection{Straightness Selector}
\label{subsec:windowed_straightness}
Our root-selection view in Sec.~\ref{sec:root-selection} allows many possible choices of the root selector $\phi_j$.
To design a practical selector that improves the inversion, we first conjecture and validate that trajectory straightness is highly correlated with the failure of fixed-point inversion, motivating a straightness-based root selector.
\paragraph{Straightness and Inversion Failure}

\begin{figure}[t]
  \centering
  \begin{subfigure}[b]{0.35\linewidth}
    \centering
    \includegraphics[width=\linewidth]{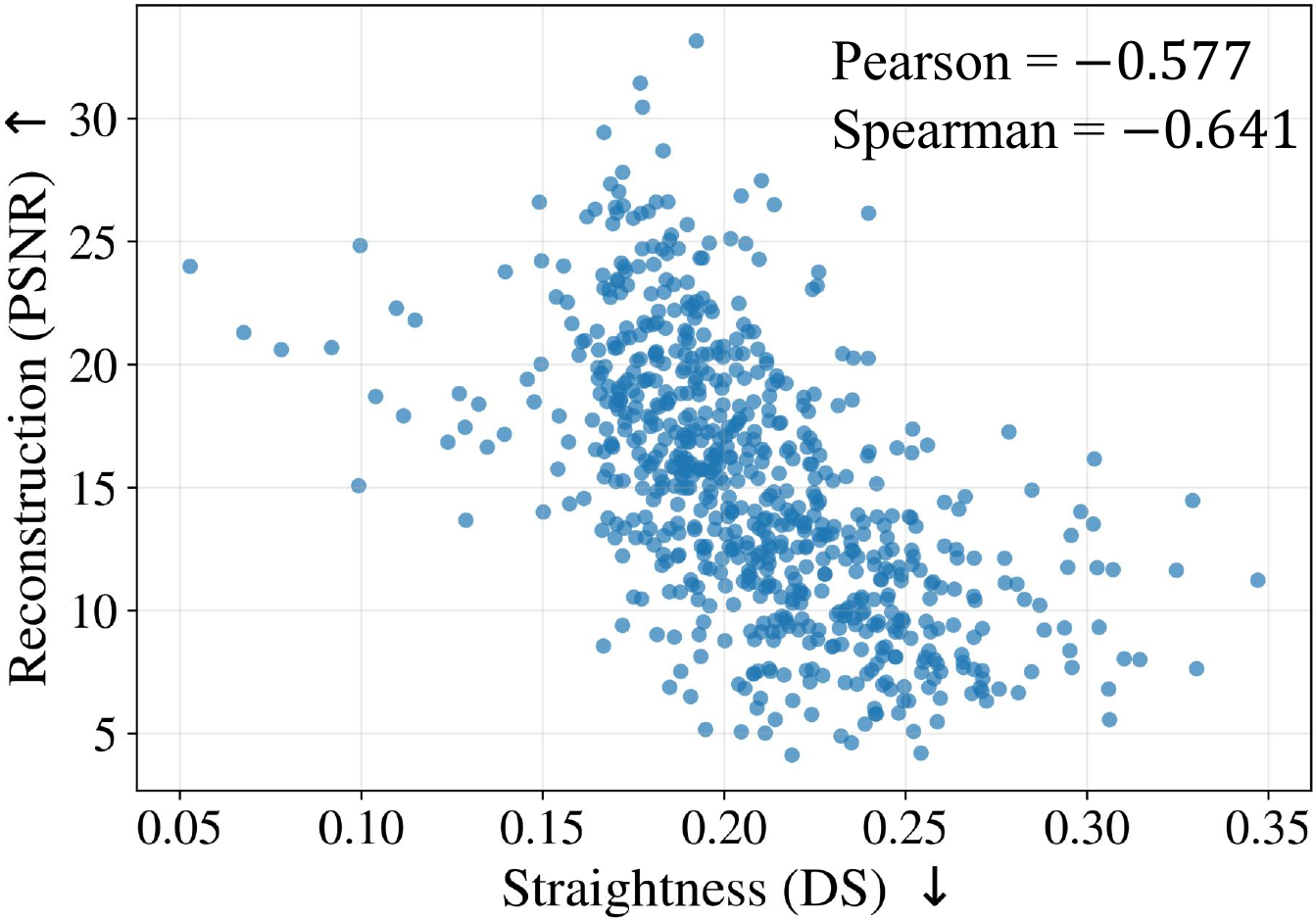}
    \subcaption{
    Reconstruction PSNR versus discrete trajectory straightness.
    }
    \label{fig:recon-straightness-correlation}
  \end{subfigure}
  \hfill
  \begin{subfigure}[b]{0.63\linewidth}
    \centering
    \includegraphics[width=\linewidth]{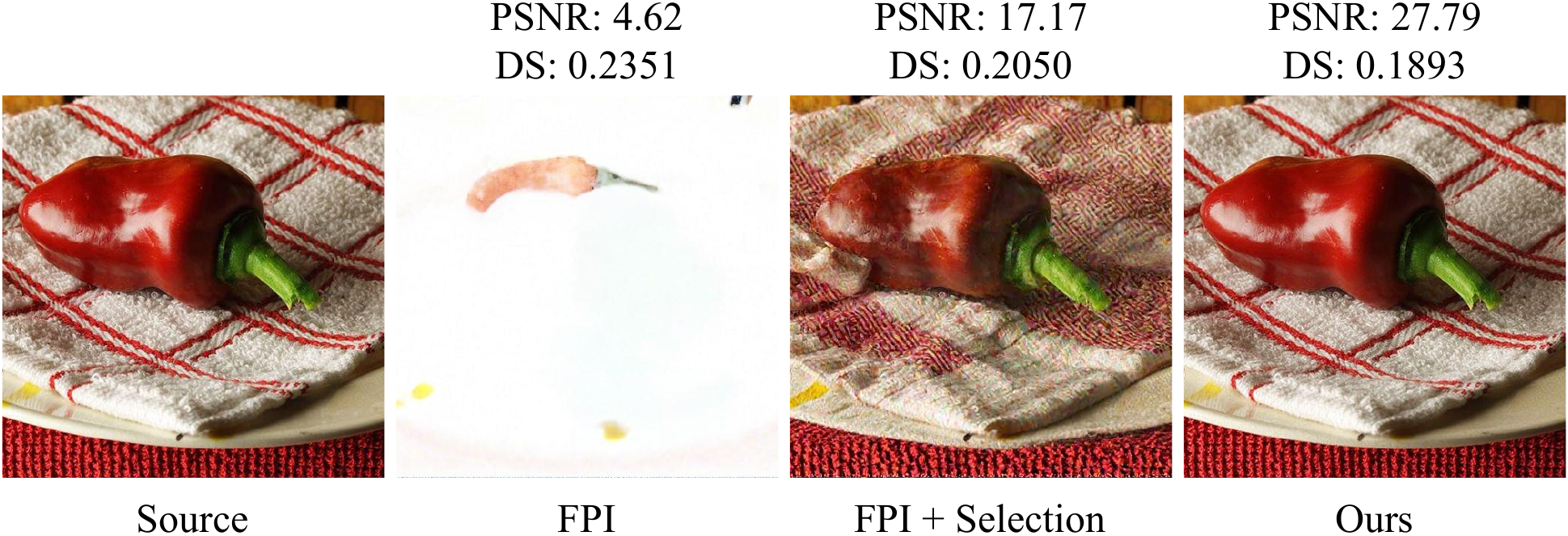}
    \subcaption{
    An image reconstruction result using straightness as a selection criterion.
    }
    \label{fig:recon-straightness-vis}
  \end{subfigure}
  \caption{
  Preliminary analysis regarding straightness. Straightness strongly correlates with reconstruction accuracy (left), and can be used as a selection criterion to enhance inversion (right).
  }
  \label{fig:recon-straightness}
\end{figure}

If fixed-point inversion were solved to exact convergence at every step, a recovered state $x_j^\star$ would satisfy $r_j(x_j^\star)=0$, so applying the same discrete generation solver would map it back to $x_{j-1}$.
In practice, however, we only run a finite number of fixed-point iterations, and the obtained state $\hat{x}_j$ generally retains a nonzero residual $r_j(\hat{x}_j)\neq0$.
During reconstruction or editing, this residual becomes a local state mismatch that is recursively propagated by later generation steps.
Notably, how severely such per-step errors accumulate depends on the local behavior of $u_\theta$ along the inverse trajectory, since the generation step re-evaluates $u_\theta$ at each replayed state (\emph{e.g.}, Eq.~\eqref{eq:euler_gen}).
We conjecture that curved trajectories are more susceptible to this accumulation because the local constant-velocity approximation of a discretized solver is less reliable, so small residuals can induce larger changes in subsequent velocity evaluations.
Thus, among approximate fixed-point roots with comparable residuals, we prefer the one inducing a straighter inverse trajectory, as it is less likely to amplify residual errors during generation.

To empirically validate this claim, we define discretized straightness (DS), a sample-wise discrete counterpart of Eq.~\eqref{eq:rf_straightness} that measures the straightness of an individual inversion trajectory:

\begin{equation}
    \label{eq:discretized_straightness}
    DS(X) = \sum_{i=1}^{T}h_i ||v_i-\bar v||^2,
\end{equation}
where $h_i=t_i-t_{i-1}$, $v_i=(x_i-x_{i-1})/h_i$, and $\bar v = {\sum h_i v_i} / {\sum h_i}$.
Similar to the continuous counterpart, a perfectly straight path would yield $DS(X)=0$.

In Fig.~\ref{fig:recon-straightness-correlation}, we plot reconstruction accuracy (PSNR) against the straightness measure (DS).
The plot shows a strong correlation between them: straighter inverse trajectories tend to yield more accurate reconstruction.
We also present an example in Fig.~\ref{fig:recon-straightness-vis} showing that straightness can serve as a valid objective for root selection.
The FPI+Selection variant, which performs best-of-$N$ selection (details in App.~\ref{app:greedy-selection-details}) to choose the straightest path among differently initialized states, can prevent the collapse of naive fixed-point iteration.
These results suggest that straightness is a strong indicator of inversion failure and can also serve as a practical root-selection criterion.
In later sections, we propose selecting such roots on the fly as the fixed-point iteration proceeds, which requires no multiple initializations and further improves reconstruction accuracy, as shown in Fig.~\ref{fig:recon-straightness-vis}.

\paragraph{Straightness Root Selector}

Motivated by the previous observation, we construct a causal root selector inspired by the discretized rectified flow straightness in Eq.~\eqref{eq:discretized_straightness}.
The global mean velocity $\bar v$ depends on the entire trajectory and is unavailable during inversion.
We therefore use the weighted average of previously observed velocities within a window of size $W$ as a causal reference:

\begin{equation}
    \bar v^{W}_{j-1} :=
    \frac{\sum_{i=1}^{W} h_{j-i} v_{j-i}}{\sum_{i=1}^{W} h_{j-i}}.
    \label{eq:causal_local_mean_velocity}
\end{equation}

Using this reference, we define an on-the-fly, per-step straightness proxy for root selection as follows.

\begin{definition}[Straightness Root Selector]
For a candidate current state $x$ at time $t_j$, the straightness root selector is defined as
\begin{equation}
    \phi^{\mathrm{str}}_{j,W}(x)
    :=
    h_j
    \left\|
        v_j(x)-\bar v^{W}_{j-1}
    \right\|^2
    =
    \frac{1}{h_j}
    \|x-a_{j,W}\|^2,
    \label{eq:causal_local_straightness_selector}
\end{equation}
which is a quadratic distance from the straightness anchor:
\begin{equation}
    a_{j,W}
    :=
    x_{j-1}+h_j\bar v^{W}_{j-1}.
    \label{eq:causal_local_straightness_anchor}
\end{equation}
\end{definition}

\begin{corollary*}
Since the selector is a quadratic distance, the minimization problem in Eq.~\eqref{eq:root_selection_general} reduces to
    \begin{equation}
        x_j^\star
        \in
        \operatorname{proj}_{S_j^C}(a_{j,W})
        :=
        \arg\min_{x\in S_j^C}
        \left\|x-a_{j,W}\right\|^2.
        \label{eq:causal_projection_selected_root}
    \end{equation}
\end{corollary*}
Although the form has changed, minimizing our root selector does not relax the exact inversion constraint.
Moreover, the criterion requires no additional model forward passes, as it is computed directly from the previous inversion trajectory.

\subsection{Root-Selecting Fixed-Point Inversion}
The projection problem in Eq.~\eqref{eq:causal_projection_selected_root} cannot be solved directly because the constraint set $S_j^C$ is only available implicitly through the fixed-point equation $x=P_j(x)$.
We therefore propose to apply the Halpern iteration~\cite{halpern1967fixed}:
\begin{equation}
    x_j^k
    =
    \alpha_k a_{j,W}
    +
    (1-\alpha_k)P_j(x_j^{k-1}),
    \label{eq:straightness_anchored_fpi}
\end{equation}
where $a_{j,W}$ is the straightness anchor defined in Eq.~\eqref{eq:causal_local_straightness_anchor}, and $\alpha_k\in(0,1)$ is a sequence that satisfies:
\begin{equation}
    \alpha_k\in(0,1),
    \qquad
    \alpha_k\to0,
    \qquad
    \sum_{k=1}^{\infty}\alpha_k=\infty,
    \qquad
    \sum_{k=2}^{\infty}|\alpha_k-\alpha_{k-1}|<\infty.
    \label{eq:vanishing_anchor_schedule}
\end{equation}

In our experiments, we consider the following family of schedules:
\begin{equation}
    \label{eqn:alpha-schedule}
    \alpha_k=\alpha_1\frac{\delta}{(k-1) + \delta},
    \qquad k\ge 1,\ \ 0<\alpha_1<1,\ \ \delta>0.
\end{equation}
The schedule is controlled by two hyperparameters: $\alpha_1$ sets the value at the initial iteration, and $\delta$ controls how fast the schedule decreases. The schedule satisfies the condition in Eq.~\eqref{eq:vanishing_anchor_schedule} (App.~\ref{app:alpha-schedule-verification}).

Under standard assumptions, the Halpern iteration converges strongly to the projection of the anchor $a_{j,W}$ onto the fixed-point set~\cite{wittmann1992approximation}.
It works by blending two terms: the fixed-point term $P_j$ enforces the original local inversion equation, while the anchor term $a_{j,W}$ biases the iteration toward a straighter trajectory.
Since $\alpha_k$ vanishes in the limit, the anchor influences which root is selected, but the limiting point, when it exists, remains an exact inverse root.

We now state the guarantee under a standard local nonexpansiveness assumption commonly used in the fixed-point inversion literature~\cite{garibi2024renoise, samuellightning}:

\begin{assumption}[Local Root-Selection setting]
\label{assump:local_root_selection}
For each inversion step $j$, there exists a closed convex set
$C_j\subset\mathbb{R}^d$ such that
\begin{equation}
    P_j(C_j)\subset C_j,
    \qquad
    a_{j,W},x_j^0\in C_j,
    \qquad
    S_j^C:=\operatorname{Fix}(P_j)\cap C_j\neq\emptyset.
    \label{eq:assump_local_root_selection}
\end{equation}
Moreover, $P_j$ is nonexpansive on $C_j$:
\begin{equation}
    \|P_j(x)-P_j(y)\|\le \|x-y\|,
    \qquad
    \forall x,y\in C_j.
    \label{eq:assump_nonexpansive}
\end{equation}
\end{assumption}

\begin{theorem}[Straightness-minimizing fixed-point selection]
\label{thm:straightness_anchored_fpi}
Under Assumption~\ref{assump:local_root_selection} and with $\alpha_k$ satisfying Eq.~\eqref{eq:vanishing_anchor_schedule}, the sequence generated by Eq.~\eqref{eq:straightness_anchored_fpi} converges to
\begin{equation}
    x_j^\star
    \in
    \arg\min_{x\in S_j^C}
    \phi^{\mathrm{str}}_{j,W}(x).
    \label{eq:anchored_straightness_minimizer}
\end{equation}
\end{theorem}

\begin{proof}
The proof is given in App.~\ref{app:anchored_fpi_proof}.
\end{proof}

\paragraph{Decoupled Momentum}

Although Eq.~\eqref{eq:straightness_anchored_fpi} gives the desired selection principle, we empirically find that smoothing the fixed-point correction improves finite-iteration behavior.
Thus, we propose a decoupled momentum, where a momentum variable is applied only to the $P_j$ term (which motivates the term \textit{decoupled}), and the straightness anchor remains unchanged.

The decoupled momentum, with initialization set to $M_j^0=x_j^0$, is:
\begin{align}
    M_j^k
    &=
    \mu M_j^{k-1}
    +
    (1-\mu)P_j(x_j^{k-1}),
    \label{eq:decoupled_momentum_state}
    \\
    x_j^k
    &=
    \alpha_k a_{j,W}
    +
    (1-\alpha_k)M_j^k,
    \label{eq:decoupled_momentum_anchor}
\end{align}
where $\mu\in[0,1)$ is the momentum coefficient that controls the strength of momentum; when $\mu=0$, it exactly recovers the Halpern iteration in Eq.~\eqref{eq:straightness_anchored_fpi}.

We posit that the decoupled momentum update yields better convergence by selectively smoothing only the update signal from $P_j(x_j^{k-1})$, which depends on iteration-varying model outputs whose errors and noise can harm convergence.
In contrast, the anchor $a_{j,W}$ remains the same throughout the iteration, so not averaging it into the momentum variable helps preserve the root-selection signal.
In the finite-iteration regime, decoupled momentum can also help damp oscillations in the inner-loop updates by retaining a component of the previous iterate that counteracts sign-reversing fixed-point proposals. We provide a local analysis and empirical oscillation statistics in App.~\ref{app:decoupled_momentum}.

Beyond this empirical gain, our decoupled momentum update also achieves exact convergence asymptotically.
The complete inversion procedure, combining the straightness anchor, the anchored update, and decoupled momentum, is summarized in Algorithm~\ref{alg:straightness_fixed_point_inversion} (App.~\ref{app:algorithm}).

\begin{theorem}[Momentum preserves the straightness-selected root]
\label{thm:momentum_preserves_selection}
Under the same assumptions as Theorem~\ref{thm:straightness_anchored_fpi}, with $\alpha_k$ in Eq.~\eqref{eqn:alpha-schedule} and with the momentum weight $ \mu \in [0,1)$, the sequence generated by Eqs.~\eqref{eq:decoupled_momentum_state}--\eqref{eq:decoupled_momentum_anchor} satisfies
\begin{equation}
    x_j^k\to x_j^\star,
    \qquad
    x_j^\star
    \in
    \arg\min_{x\in S_j^C}
    \phi^{\mathrm{str}}_{j,W}(x).
    \label{eq:momentum_straightness_minimizer}
\end{equation}
\end{theorem}
\begin{proof}
The proof is given in App.~\ref{app:momentum_proof}.
\end{proof}

\section{Related Work}
\paragraph{Inversion}
Inversion is central to training-free image editing, but naive methods such as DDIM inversion~\cite{songdenoising-ddim} accumulate discretization error.
Prior diffusion inversion methods reduce this error by modifying conditioning~\cite{miyake2023negative}, optimizing auxiliary variables~\cite{dong2023pti,mokady2023nulltext}, or coupling inversion and reconstruction streams~\cite{wallace2023edict}.
Fixed-point inversion instead solves the discrete local inverse equation directly; ReNoise~\cite{garibi2024renoise}, AIDI~\cite{pan2023aidi}, and GNRI~\cite{samuellightning} improve this process through iterative noising, acceleration, or Newton-style updates.
These methods focus on obtaining a consistent inverse point, while SelFix addresses the complementary problem of selecting among distinct approximate roots.
\paragraph{Rectified flow inversion and trajectory straightness}
For rectified flow models, recent inversion methods such as RF-Inversion~\cite{rout2024semantic}, RF-Solver~\cite{wang2025taming-rfsolver}, FireFlow~\cite{deng2025fireflow}, UniEdit-Flow~\cite{jiao2026unieditflow}, and PMI~\cite{wang2026freelunch} reduce inversion error through tailored solvers, optimal-control formulations, or velocity priors.
SelFix differs by preserving the fixed-point inverse equation and using trajectory straightness to select the preferred root.
This is motivated by the design principle of rectified flow, where straighter trajectories are easier to simulate with coarse solvers~\cite{liu2022rectifiedflow}; related work improves straightness through coupling design, curvature or variance modeling, encoders, or test-time refinement~\cite{pooladian2023multisample,tong2023minibatchot,lee2023trajcurve,guo2025vrm,kim2024sfno,cha2026training}.
We instead use straightness as a causal root-selection criterion during inversion.
\paragraph{Fixed-point selection}
Halpern iteration is a classical fixed-point selection mechanism for nonexpansive maps~\cite{halpern1967fixed,bauschke1996approximation,wittmann1992approximation}.
Our contribution is not a new general fixed-point algorithm, but an inversion-specific instantiation: we derive a rectified flow straightness anchor and show that the resulting anchored fixed-point update selects the corresponding exact local inverse root asymptotically.

\section{Experiments}
Following the convention~\cite{deng2025fireflow}, we validate the effectiveness of our method in real image reconstruction and editing tasks.
For baselines, we consider methods using flow-tailored higher-order inversion solvers (RF-Solver~\cite{wang2025taming-rfsolver} and FireFlow~\cite{deng2025fireflow}), and methods using fixed-point formulation (ReNoise~\cite{garibi2024renoise} and AIDI~\cite{pan2023aidi}).
We also include naive baselines, one that simply integrates the flow in reverse direction (ReFlow) and one that does a naive fixed-point iteration (FPI).
We use the FLUX.1-dev~\cite{BlackForestLabs2024FLUX1} model as a backbone, and the number of function evaluations (NFE) is matched across baselines for fair comparison.
Additional experiment details are provided in App.~\ref{app:exp-details}, along with further results using different backbones, editing strategies, or inversion baselines in App.~\ref{app:additional-results}.

\subsection{Image Reconstruction}

\begin{table}[t]
\centering
\footnotesize
\setlength{\tabcolsep}{9pt}
\caption{Quantitative image reconstruction results on PIE-Bench.}
\label{tab:reconstruction}
\begin{tabular}{lcccccc}
\toprule
\textbf{Method}
& \textbf{LPIPS} $\downarrow$
& \textbf{SSIM} $\uparrow$
& \textbf{PSNR} $\uparrow$
& \textbf{DS} $\downarrow$
& \textbf{$\sum \phi^{\mathrm{str}}_{j,W}$ $\downarrow$}
& \textbf{NFE} \\
\midrule
ReFlow & 0.2382 & 0.6965 & 19.54 & 0.181 & 0.176 & 166 \\
\midrule
\rowcolor{gray!10}
\multicolumn{7}{l}{\textit{Improved Solver}} \\
\addlinespace[2pt]
RF-Solver & 0.1997 & 0.7400 & 21.31 & 0.172 & 0.162 & 168 \\
FireFlow  & 0.1037 & 0.8408 & 26.98 & 0.185 & 0.180 & 168 \\
\midrule
\rowcolor{gray!10}
\multicolumn{7}{l}{\textit{Fixed-Point Inversion}} \\
\addlinespace[2pt]
ReNoise & 0.1155 & 0.8159 & 22.97 & 0.177 & 0.157 & 165 \\
AIDI    & 0.1533 & 0.7862 & 23.00 & 0.187 & 0.165 & 165 \\
FPI     & 0.3059 & 0.6317 & 15.05 & 0.208 & 0.186 & 165 \\
\addlinespace[2pt]
\textbf{SelFix (Ours)}
& \textbf{0.0742}
& \textbf{0.8714}
& \textbf{28.29}
& \textbf{0.170}
& \textbf{0.151}
& 165 \\
\bottomrule
\end{tabular}
\end{table}

\begin{figure}[t]
\vspace{-1em}
\centering
\includegraphics[width=0.95\linewidth]{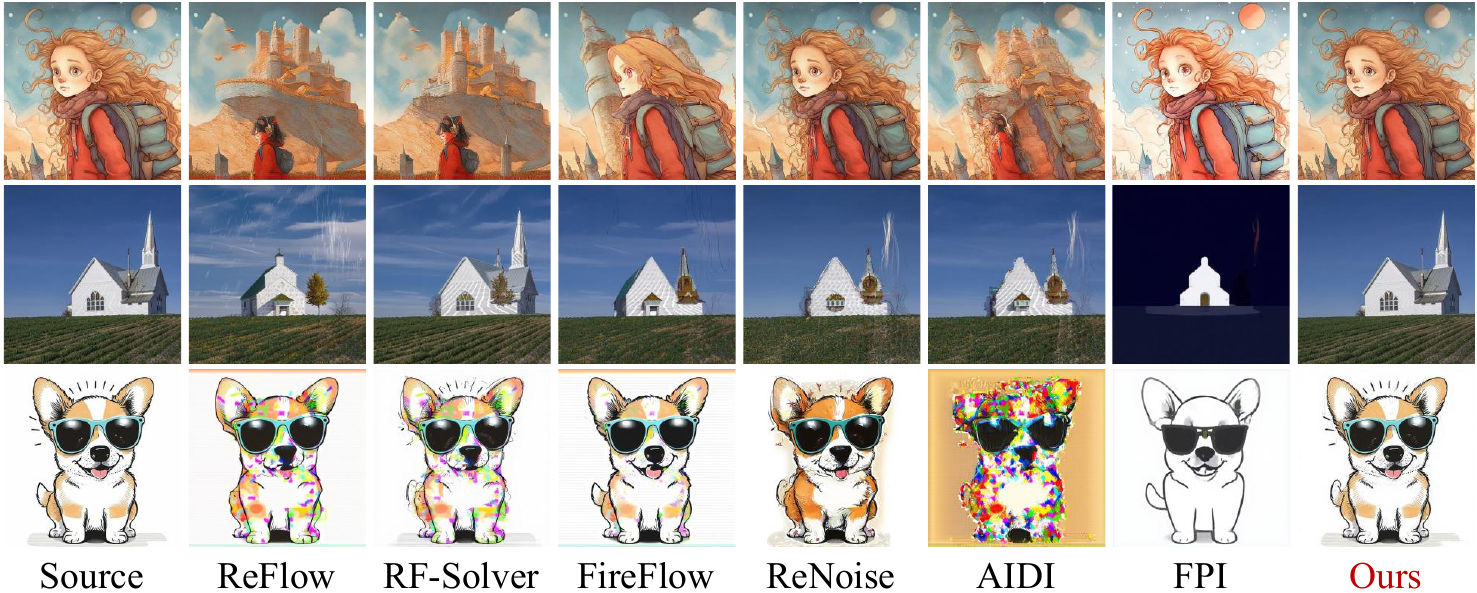}
\caption{Qualitative real-image reconstruction results on PIE-Bench.
More results are in App.~\ref{app:additional-reconstruction}.}
\label{fig:qual-recon}
\vspace{-1em}
\end{figure}

\paragraph{Evaluation Metrics}
We report the reconstruction results on PIE-Bench~\cite{ju2023direct} using standard reconstruction accuracy metrics.
We also report inversion trajectory straightness in two ways:
the discretized trajectory straightness $DS$ in Eq.~\eqref{eq:discretized_straightness}, and the accumulated selector value $\sum_j \phi^{\mathrm{str}}_{j,W}$, where $\phi^{\mathrm{str}}_{j,W}$ is the straightness proxy used by our root selector in Eq.~\eqref{eq:causal_local_straightness_selector}.

\paragraph{Results}
As shown in Tab.~\ref{tab:reconstruction}, SelFix achieves the best performance across all reconstruction metrics, outperforming both advanced fixed-point inversion methods and naive FPI.
Qualitative results in Fig.~\ref{fig:qual-recon} show the same trend: SelFix avoids the error accumulation that causes other fixed-point baselines to fail.
This suggests that fixed-point inversion benefits from selecting among reachable local inverse roots, rather than simply applying unguided fixed-point iteration.

The straightness measurements further support this interpretation.
SelFix obtains the lowest $DS$ and the lowest accumulated proxy value $\sum_j\phi^{\mathrm{str}}_{j,W}$ among all compared methods.
In particular, compared with FPI, SelFix reduces $\sum_j\phi^{\mathrm{str}}_{j,W}$ from $0.1855$ to $0.1505$, which is expected since the proposed update explicitly biases the fixed-point iteration toward roots close to the straightness anchor.
This reduction is accompanied by a decrease in the trajectory-level straightness measure $DS$ from $0.2080$ to $0.1702$.
That is, although $\phi^{\mathrm{str}}_{j,W}$ is a proxy, its reduction is well aligned with the observed decrease in $DS$.
Taken together, the improvement in both reconstruction metrics and straightness metrics validates our method and supports the hypothesis that selecting a straighter inverse trajectory helps reduce error accumulation during reconstruction.

\subsection{Image Editing}

\paragraph{Evaluation}

We evaluate SelFix on the PIE-Bench~\cite{ju2023direct} prompt-based image editing task.
In this task, we first invert a source image into noise and then generate an edited image from the inverted noise using a target prompt.
We use the same editing strategy across all inversion methods to ensure a fair comparison.
We report Structure Distance (DINO self-similarity), background preservation (PSNR, LPIPS, and SSIM), and CLIP similarity on the whole image and edited region (Tab.~\ref{tab:editing}).
Since source preservation and target-prompt alignment exhibit a trade-off~\cite{deng2025fireflow}, we plot CLIP similarity against background PSNR in Fig.~\ref{fig:pareto-edit} to compare their preservation-alignment trade-offs.
We further analyze this trade-off under different editing strategies in App.~\ref{app:tradeoff}.

\paragraph{Results}
As shown in Tab.~\ref{tab:editing}, SelFix achieves the strongest source preservation among all compared methods.
This is consistent with the reconstruction results: by selecting a straighter inverse trajectory, SelFix obtains a more faithful inverse noise, which provides a better starting point for preserving source content during editing.
SelFix obtains the lowest structure distance and the best background-preservation scores, while remaining competitive in target-prompt alignment.
Overall, it improves the preservation-editing trade-off, as shown in Fig.~\ref{fig:pareto-edit}.

\begin{figure*}[!t]
\centering
\footnotesize
\begin{minipage}{0.3\textwidth}
    \centering
    \includegraphics[width=0.99\linewidth]{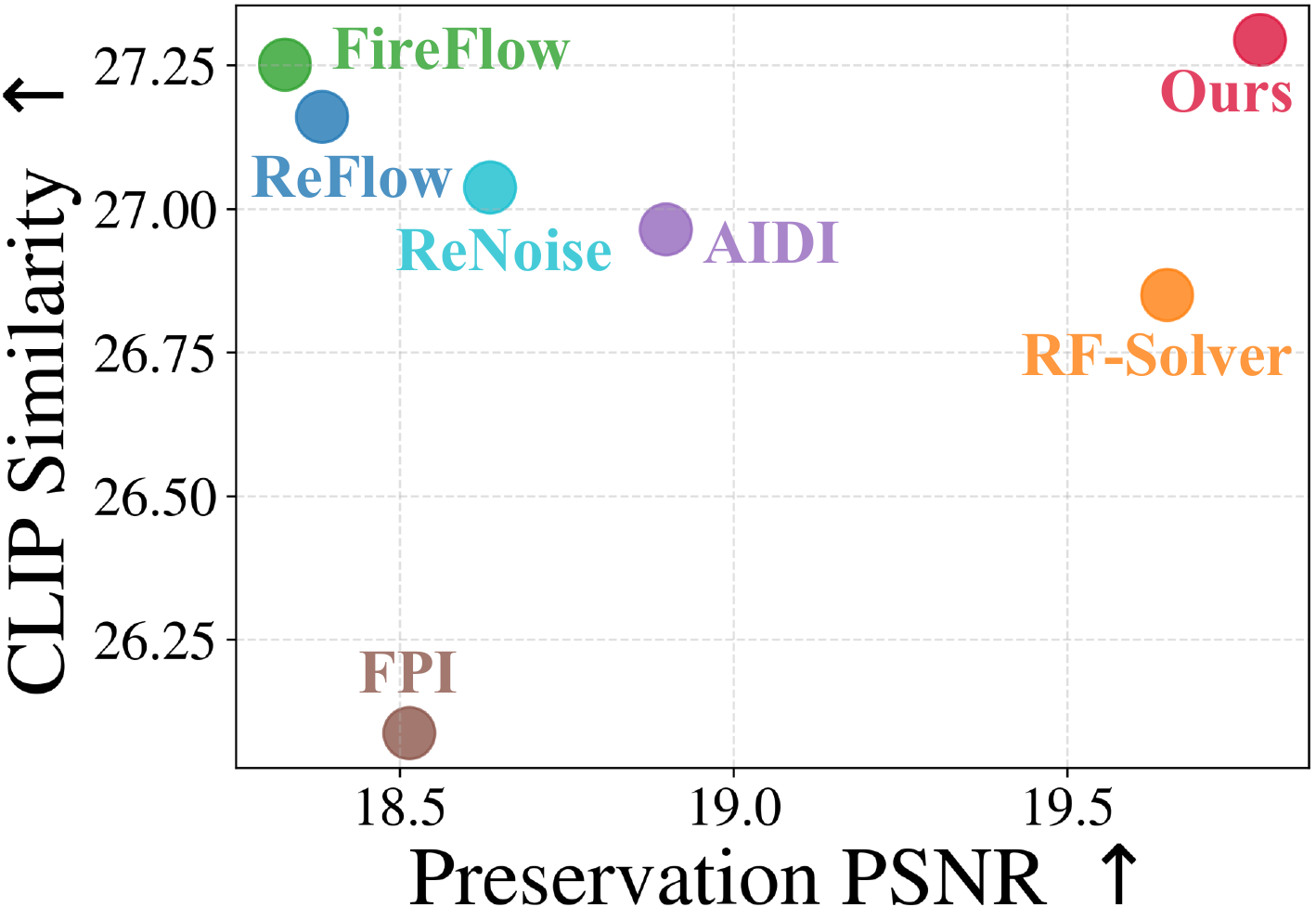}
    \captionof{figure}{Source-preservation and prompt-alignment trade-off on PIE-Bench editing.}
    \label{fig:pareto-edit}
\end{minipage}
\hfill
\begin{minipage}{0.69\textwidth}
    \centering
    \setlength{\tabcolsep}{3.8pt}
    \captionof{table}{Quantitative prompt-based editing results on PIE-Bench.}
    \label{tab:editing}
    \begin{tabular}{@{}lcccccc@{}}
    \toprule
    \multirow{2}{*}{\textbf{Method}} & \multirow{2}{*}{\textbf{\shortstack{Structure \\ Dist. $\downarrow$}}} & \multicolumn{3}{c}{\textbf{Background Preservation}} & \multicolumn{2}{c}{\textbf{CLIP Sim. $\uparrow$}} \\ \cmidrule(lr){3-5} \cmidrule(lr){6-7}
     & & \textbf{PSNR} $\uparrow$ & \textbf{LPIPS} $\downarrow$ & \textbf{SSIM} $\uparrow$ & \textbf{Whole} & \textbf{Edited} \\ \midrule

    ReFlow & 0.0657 & 18.38 & 0.2329 & 0.7204 & 27.16 & 23.85 \\ \midrule

    RF-Solver & 0.0525 & 19.65 & 0.1992& 0.7527 & 26.85 & 23.64 \\
    FireFlow & 0.0629 & 18.33 & 0.2218 & 0.7305 & 27.25 & \textbf{23.98} \\ \midrule

    ReNoise & 0.0530 & 18.64 & 0.1943  & 0.7449 & 27.04 & 23.71 \\
    AIDI & 0.0505 & 18.90 & 0.1943 & 0.7581 & 26.96 & 23.64 \\
    FPI & 0.0570 & 18.51 & 0.1830 & 0.7482 & 26.09 & 22.95 \\
    \midrule
    \textbf{SelFix (Ours)} & \textbf{0.0457} & \textbf{19.79} & \textbf{0.1767} & \textbf{0.7748} & \textbf{27.29} & 23.87 \\ \bottomrule
    \end{tabular}
\end{minipage}
\par\medskip
\centering
\includegraphics[width=\linewidth]{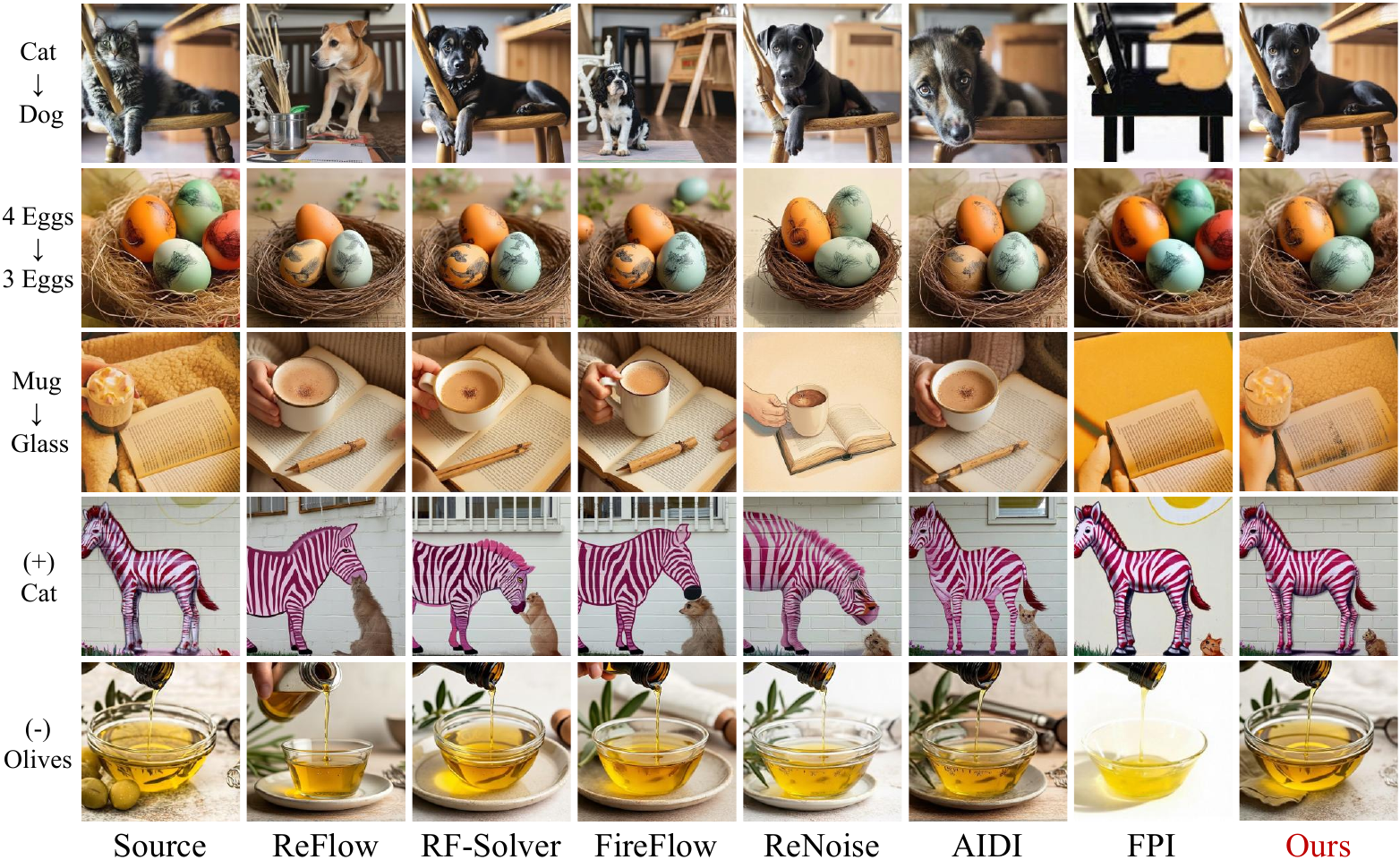}
\captionof{figure}{Qualitative prompt-based editing results on PIE-Bench.
SelFix better preserves background while applying the requested semantic edit.
We also provide more samples in App.~\ref{app:additional-editing}.}
\label{fig:qual-edit}
\vspace{-1em}
\end{figure*}

Qualitative results in Fig.~\ref{fig:qual-edit} show the same trend.
SelFix better preserves background details and global structure while still applying the requested edit.
For example, in the second row, SelFix preserves the nest shape and background while correctly changing four eggs to three.
In contrast, baselines either alter the global image structure or fail to render the requested number of eggs.

\subsection{Ablation Study}

\paragraph{Guiding Mechanism}
Tab.~\ref{tab:heuristic-regularization} compares SelFix with alternative ways of injecting the straightness anchor.
Specifically, we compare SelFix with variants using either constant scheduling or gradient-based correction to minimize MSE to the anchor.
As can be seen, all variants produce degraded results in terms of PSNR.
While blending the straightness anchor with a large constant $\alpha$ or using a gradient-based update with a large learning rate can lead to even smaller DS, they bias the fixed-point equation too much and result in inaccurate inversion.
Conversely, when hyperparameters are tuned to match DS across methods, we see that SelFix performs the best.
We attribute this to SelFix's vanishing anchor schedule, which simultaneously guides the iteration toward a straighter trajectory and produces an accurate fixed point.

\begin{table*}[t]
\vspace{-0.5em}
\begin{minipage}[t]{0.48\textwidth}
\centering
\footnotesize
\captionof{table}{Ablation on guiding mechanism.}
\label{tab:heuristic-regularization}
\begin{tabular}{@{}lcc@{}}
  \toprule
  \textbf{Method} & \textbf{PSNR} $\uparrow$ & \textbf{DS} $\downarrow$ \\
  \midrule
  Iteration with constant $\alpha$ (0.5) & 15.53 & \textbf{0.127} \\
  Iteration with constant $\alpha$ (0.0068) & 27.86 & 0.170 \\
  Gradient descent (lr=0.1) & 19.56 & 0.157 \\
  Gradient descent (lr=0.01) & 27.81 & 0.170 \\
  \textbf{SelFix (Ours)} & \textbf{28.29} & 0.170 \\
  \bottomrule
\end{tabular}
\end{minipage}
\hfill
\begin{minipage}[t]{0.49\textwidth}
\centering
\footnotesize
\captionof{table}{Ablation on the choice of selector anchor.}
\label{tab:abl-selector}
\setlength{\tabcolsep}{4pt}
\begin{tabular}{lccc}
\toprule
\textbf{Selector Anchor} & \textbf{LPIPS} $\downarrow$ & \textbf{SSIM} $\uparrow$ & \textbf{PSNR} $\uparrow$ \\
\midrule
Zero         & 0.6974 & 0.3395 & 13.04 \\
Random       & 0.1779 & 0.7598 & 23.69 \\
\textbf{Straightness (Ours)}
             & \textbf{0.0742} & \textbf{0.8714} & \textbf{28.29} \\
\bottomrule
\end{tabular}
\end{minipage}
\end{table*}

\paragraph{Selector Choice}
We examine the importance of the straightness selector by replacing its anchor with zero or random anchors while keeping the remaining inversion procedure unchanged.
As shown in Tab.~\ref{tab:abl-selector}, both alternatives substantially degrade reconstruction performance, with PSNR dropping from 28.29 to 13.04 and 23.69, respectively.
These results indicate that the improvement of SelFix arises not simply from the anchored update itself, but from carefully choosing an anchor informed by trajectory straightness.

\begin{table*}[t]
\centering
\footnotesize
\caption{Hyperparameter ablations.}
\label{tab:abl-hyperparameters}
\vspace{-3pt}
    \begin{subtable}[t]{0.42\linewidth}
        \centering
        \footnotesize
        \subcaption{Momentum coefficient $\mu$.}
        \label{tab:abl-momentum}
        \begin{tabular}{@{}cccc@{}}
        \toprule
        \textbf{Decoupled?} & $\boldsymbol{\mu}$ & \textbf{PSNR} $\uparrow$ & \textbf{DS} $\downarrow$ \\
        \midrule
        \checkmark
        & 0.75 & 26.93 & \textbf{0.126} \\
        \checkmark & \textbf{0.5} & \textbf{28.29} & 0.170 \\
        \checkmark & 0.25 & 23.71 & 0.180 \\
        $\times$
        & 0.5 & 27.54 & 0.170 \\
        --
        & 0.0 & 16.77 & 0.204 \\
        \bottomrule
        \end{tabular}
    \end{subtable}\hfill%
    \begin{subtable}[t]{0.32\linewidth}
        \centering
        \footnotesize
        \subcaption{Decay $\delta$.}
        \label{tab:abl-decay}
        \begin{tabular}{@{}ccc@{}}
        \toprule
        $\boldsymbol{\delta}$ & \textbf{PSNR} $\uparrow$ & \textbf{DS} $\downarrow$ \\
        \midrule
        1 & 27.29 & \textbf{0.1678} \\
        1/2 & 27.60 & 0.1689 \\
        1/4 & 28.21 & 0.1700 \\
        \textbf{1/8} & \textbf{28.29} & 0.1702 \\
        1/16 & 27.98 & 0.1700 \\
        \bottomrule
        \end{tabular}
    \end{subtable}\hfill%
    \begin{subtable}[t]{0.26\linewidth}
        \centering
        \footnotesize
        \subcaption{Window $W$.}
        \label{tab:abl-window}
        \begin{tabular}{@{}ccc@{}}
        \toprule
        $\boldsymbol{W}$ & \textbf{PSNR} $\uparrow$ & \textbf{DS} $\downarrow$ \\
        \midrule
        1 & \textbf{28.29} & 0.1702 \\
        2 & \textbf{28.29} & 0.1701 \\
        3 & 28.26 & 0.1699 \\
        4 & 28.17 & \textbf{0.1697} \\
        \bottomrule
        \end{tabular}
    \end{subtable}
\vspace{-1em}
\end{table*}

\paragraph{Hyperparameters}
Tab.~\ref{tab:abl-hyperparameters} ablates the main hyperparameters of SelFix.
Momentum is critical for finite-iteration convergence: removing it drops PSNR to $16.77$, while $\mu=0.5$ gives the best reconstruction.
Smaller momentum insufficiently smooths the fixed-point correction, whereas larger momentum slows convergence and can leave the trajectory overly anchor-biased, yielding lower DS but worse PSNR.
Decoupled momentum also outperforms coupled momentum, supporting our choice to smooth only the fixed-point term while applying the straightness anchor directly.
For the anchor schedule, $\delta=1/8$ performs best: too fast decay gives the anchor too little time to steer root selection, while too slow decay leaves a non-negligible anchor term near the end of the inner loop and can hinder accurate inversion.

Finally, short straightness windows work best.
Although larger $W$ slightly reduces DS, it relies on longer recovered trajectory segments that may already contain accumulated numerical or model errors.
We therefore use the local setting $W=1$ by default, which provides a reliable constant-velocity estimate and the best reconstruction accuracy.

\paragraph{Robustness to T2I Backbone and Dataset}
To examine the generality of SelFix beyond FLUX.1-dev and PIE-Bench, we evaluate it with a different backbone (Stable Diffusion 3.5 Large~\cite{esser2024sd3}) and on a different dataset (COCO~\cite{lin2014mscoco}).
SelFix consistently outperforms all compared inversion baselines in LPIPS and PSNR, and also achieves the best source preservation PSNR with comparable target CLIP similarity in the SD3.5 editing experiment, reproducing the trend observed with FLUX.1-dev and PIE-Bench.
Results and corresponding experiment details are shown in App.~\ref{app:Robustness}.

\section{Conclusion}

We presented SelFix, a straightness-guided fixed-point inversion method for rectified flows that reframes inversion as root selection over the exact local inverse set.
By deriving a causal straightness anchor from the inverse trajectory and using a vanishing anchored iteration with decoupled momentum, SelFix selects a straighter exact inverse root while preserving the original fixed-point equation asymptotically.
Experiments on reconstruction and prompt-based editing show that this selection principle improves reconstruction fidelity, trajectory straightness, and source preservation over solver-based and fixed-point baselines, suggesting that trajectory geometry is an effective criterion for resolving multi-root ambiguity in rectified flow inversion.

\begin{ack}
This work was in part supported by the National Research Foundation of Korea (RS-2024-00351212 and RS-2024-00436165), and the Institute of Information \& communications Technology Planning \& Evaluation (IITP) (RS-2024-00509279, RS-2022-II220926, RS-2022-II220959, RS-2026-25518317, and RS-2019-II190075).

\end{ack}

\bibliography{neurips_2026}
\bibliographystyle{acl_natbib}

\newpage
\appendix
\appendix

\section{Limitations and Broader Impacts}
\label{sec:app-limitations}

SelFix is designed as a root-selection mechanism for fixed-point inversion in rectified flows, and therefore inherits the limitations of the underlying generative backbone.
For example, the demonstrated editing pipeline may still fail on challenging edits when the backbone model itself has limited generation or instruction-following capability.
In addition, our theoretical guarantee is local and asymptotic: it assumes a nonempty local fixed-point set and local nonexpansiveness of the fixed-point map, while practical inversion uses only a finite number of inner iterations.
Although we find that a small number of iterations, at most 10 in our experiments, is sufficient for reconstruction and editing, finite-iteration residuals cannot be fully eliminated in practice.
A finite-distance counterpart is provided in App.~\ref{app:approximate-iterate-bound}, which shows that an iterate within $\varepsilon$ of the selected root attains a residual of at most $2\varepsilon$ and a straightness gap of order $\varepsilon$.
Thus, the number of inner iterations may need to be adjusted depending on the compute budget and desired reconstruction accuracy.
The theory should therefore be interpreted as characterizing the limiting root selected by the update, rather than guaranteeing exact convergence under every finite compute budget.

\paragraph{Broader Impacts and Safeguards}
SelFix improves inversion for rectified-flow text-to-image models, which may benefit reconstruction and source-preserving image editing in creative and assistive visual applications.
At the same time, improved inversion and editing fidelity may inherit misuse risks from the underlying generative model, such as deceptive image manipulation or unauthorized modification of visual content.
These risks are not unique to SelFix, as our method does not introduce a new generative backbone, training dataset, or user-facing deployment system.
The method is evaluated using existing pretrained models and benchmarks, and the released code is limited to the inversion algorithm rather than a newly trained image generator.
Therefore, safeguards for the underlying generative model, including responsible access policies, content filtering, provenance tracking, and usage restrictions, remain the primary mechanisms for mitigating potential misuse.

\section{Algorithm}
\label{app:algorithm}

Algorithm~\ref{alg:straightness_fixed_point_inversion} states the complete SelFix inversion procedure. The outer loop over inversion steps $j$ builds the straightness anchor $a_{j,W}$ from previously recovered velocities, and the inner fixed-point loop applies the anchored update with decoupled momentum to select the straightness-minimizing root.

\begin{algorithm}[t]
\small
\caption{SelFix: Straightness-selected fixed-point inversion}
\label{alg:straightness_fixed_point_inversion}
\begin{algorithmic}[1]
\Require Image latent $x_0$; timesteps $0=t_0<t_1<\cdots<t_N$ with $h_j=t_j-t_{j-1}$; velocity field $u_\theta(\cdot,t,c)$; conditioning $c$; fixed-point iterations $K$; window $W>0$; Schedule $\{\alpha_k\}_{k=1}^K$ (Eq.~\eqref{eqn:alpha-schedule}); momentum $\mu\in[0,1)$.
\Ensure Inverted latent $x_N$ and inversion trajectory $\{x_j\}_{j=0}^N$.
\For{$j=1,\ldots,N$} \Comment{outer loop: inversion steps}
    \State Define $r_j(z)\coloneqq f_{t_j\to t_{j-1}}(z)-x_{j-1}$ and $P_j(z)\coloneqq z-r_j(z)$. \Comment{Eqs.~\eqref{eq:inversion_residual},~\eqref{eq:fixed_point_map}}
    \Statex \hfill\textit{(Euler reverse step: $P_j(z)=x_{j-1}+h_j\,u_\theta(z,t_j,c)$)}
    \State $m_j\gets \min\{W,\,j-1\}$.
    \State $z^0\gets x_{j-1}$.
    \State $M^0\gets z^0$.
    \If{$j\neq1$}
        \State $\bar v^{\,W}_{j-1}\gets \bigl(\textstyle\sum_{\ell=j-m_j}^{j-1} h_\ell\,\widehat v_\ell\bigr)\bigl/\bigl(\textstyle\sum_{\ell=j-m_j}^{j-1} h_\ell\bigr)$. \Comment{Eq.~\eqref{eq:causal_local_mean_velocity}}
        \State $a_{j,W}\gets x_{j-1}+h_j\,\bar v^{\,W}_{j-1}$. \Comment{Eq.~\eqref{eq:causal_local_straightness_anchor}}
    \EndIf
    \For{$k=1,\ldots,K$} \Comment{inner loop: fixed-point iteration}
        \State $q^k\gets P_j(z^{k-1})$.
        \State $M^k\gets \mu M^{k-1}+(1-\mu)\,q^k$. \Comment{Eq.~\eqref{eq:decoupled_momentum_state}}
        \If{$j=1$}
            \State $z^k\gets M^k$.
        \Else
            \State $z^k\gets \alpha_k a_{j,W}+(1-\alpha_k)M^k$.
            \Comment{Eq.~\eqref{eq:decoupled_momentum_anchor}}
        \EndIf
    \EndFor
    \State $x_j\gets z^K$; \ \ $\widehat v_j\gets (x_j-x_{j-1})/h_j$.
\EndFor
\State \Return $x_N$ and $\{x_j\}_{j=0}^N$.
\end{algorithmic}
\end{algorithm}

\section{Details for Preliminary Analyses}
\label{app:initial-analysis}

\subsection{Multiplicity of Fixed Points}
\label{app:multiroot-details}

For the analysis in Fig.~\ref{fig:multiroot}, we use 16 different initializations, where each point is obtained by adding different Gaussian noise to the previous state $x_{j-1}$.
The purpose of the analysis is to test whether the learned and discretized local inverse equation at a fixed inversion step admits multiple approximate fixed-point roots in practice.
We use the FLUX.1-dev~\cite{BlackForestLabs2024FLUX1} model with $T=40$ denoising steps, and conduct the analysis at the middle of denoising ($j=20$).
To prevent external error (\emph{i.e.}, error from the latent autoencoder or a mismatch between the real image distribution and the model-generated image distribution), we conduct the analysis on a synthesized image.
Generation uses the prompt \textit{a corgi wearing a blue bandana sitting on a park path}, and inversion uses $K=15$ iterations with Anderson acceleration to boost convergence.

At fixed-point iteration $k$, we report two quantities. The first is the mean fixed-point residual,
\begin{equation}
    R_k
    :=
    \frac{1}{N_{\mathrm{init}}}
    \sum_{n=1}^{N_{\mathrm{init}}}
    \left\|P_j(x^{k}_{j,n})-x^{k}_{j,n}\right\|_2 .
    \label{eq:app-multiroot-residual}
\end{equation}
The second is the mean pairwise distance among the current states,
\begin{equation}
    D_k
    :=
    \frac{2}{N_{\mathrm{init}}(N_{\mathrm{init}}-1)}
    \sum_{1\le n<m\le N_{\mathrm{init}}}
    \left\|x^{k}_{j,n}-x^{k}_{j,m}\right\|_2 .
    \label{eq:app-multiroot-pairwise}
\end{equation}
A small residual $R_k$ means that each run is close to satisfying the same learned local inverse equation. If, at the same time, $D_k$ remains much larger than $R_k$, then the converged states are not merely small numerical perturbations of a single point. Instead, they are distinct approximate fixed-point roots of the same learned and discretized local equation. This is the empirical phenomenon shown in Fig.~\ref{fig:multiroot}.

\paragraph{Extended Analysis Across Backbones and Datasets}
To strengthen the observation in Fig.~\ref{fig:multiroot}, we extend the analysis from a single illustrative example to the first 20 PIE-Bench prompts with both FLUX.1-dev and SD3.5 Large, and further repeat it on COCO with FLUX.1-dev. For each input, we run the same local fixed-point iteration from multiple initializations and measure the mean residual and mean pairwise distance between the resulting iterates. The results are shown in Tab.~\ref{tab:fp_multiplicity}.

\begin{table}[t]
\centering
\caption{Extended analysis of approximate fixed-point multiplicity. Mean fixed-point residual and mean pairwise $L_2$ distance between iterates from different initializations, at inner fixed-point iterations 1, 8, and 15.}
\label{tab:fp_multiplicity}
\small
\begin{tabular}{llccc}
\toprule
Backbone & Dataset & FPI iteration & Mean residual & Mean pairwise $L_2$ distance \\
\midrule
\multirow{3}{*}{FLUX.1-dev} & \multirow{3}{*}{PIE-Bench}
  & 1  & 251.03 & 362.20 \\
  & & 8  & 0.1759 & 6.963 \\
  & & 15 & 0.1638 & 6.956 \\
\midrule
\multirow{3}{*}{SD3.5 Large} & \multirow{3}{*}{PIE-Bench}
  & 1  & 251.39 & 364.41 \\
  & & 8  & 2.596  & 36.61 \\
  & & 15 & 1.580  & 36.70 \\
\midrule
\multirow{3}{*}{FLUX.1-dev} & \multirow{3}{*}{COCO}
  & 1  & 251.08 & 362.11 \\
  & & 8  & 0.1733 & 5.738 \\
  & & 15 & 0.1617 & 5.731 \\
\bottomrule
\end{tabular}
\end{table}

Across all three settings, the fixed-point residual decreases by more than two orders of magnitude, whereas the pairwise distance between solutions remains large and nearly unchanged in the late iterations. Thus, the converged states are distinct approximate solutions of the same learned, discretized local inverse equation, indicating that the observed approximate-root multiplicity is not specific to a single example, backbone, or dataset.

\subsection{Greedy Multi-Start Straightness Selection}
\label{app:greedy-selection-details}
This section describes the FPI+Selection sanity check shown in Fig.~\ref{fig:recon-straightness-vis}. This is not the proposed method, but a diagnostic experiment used to test whether selecting a straighter fixed-point trajectory can improve reconstruction.

The FPI+Selection variant performs best-of-$N$ selection at each denoising step. Given the previously selected states, we generate $N_{\mathrm{cand}}=32$ candidate inverse states by running independent fixed-point iterations from differently perturbed initializations, using the same fixed-point budget as the naive FPI baseline. The perturbation distribution is
\begin{equation}
    x^0_{\mathrm{cand},n}=x^0_{\mathrm{base}}+\sigma_{\mathrm{cand}}\epsilon_n,
    \qquad
    \epsilon_n\sim\mathcal N(0,I),
    \qquad
    n=1,\ldots,N_{\mathrm{cand}} .
\end{equation}
The perturbation scale $\sigma_{\mathrm{cand}}$ is set to one, and this injection is applied at every denoising step, at the start of the inner fixed-point iteration.
After obtaining the candidate states, we evaluate the causal straightness proxy $\phi^{\mathrm{str}}_{j,W}$ in Eq.~\eqref{eq:causal_local_straightness_selector} for each candidate and select the one with the smallest value. The selected state is then used to proceed to the next denoising step.
The resulting algorithm is a zero-order greedy search, provided as empirical motivation for root selection.
Although inefficient and imperfect, this search illustrates that, among multiple feasible fixed-point inversion paths, choosing a straighter path can help produce a better reconstruction.

\section{Proofs and Mathematical Details}
\label{app:proofs}

\subsection{Notation and Standard Convergence Facts}
\label{app:proof-notation}

All proof arguments are local to one inversion step. Fix $j\in\{1,\ldots,T\}$. To avoid conflict with the total number of timesteps $T$, we write
\begin{equation}
    P:=P_j,
    \qquad
    C:=C_j,
    \qquad
    S:=S_j^C=\operatorname{Fix}(P_j)\cap C_j,
    \qquad
    a:=a_{j,W}.
    \label{eq:app-proof-shorthand}
\end{equation}
The step index $j$ is omitted whenever this causes no ambiguity. Under Assumption~\ref{assump:local_root_selection}, $C$ is closed and convex, $P(C)\subset C$, $P$ is nonexpansive on $C$, and $S$ is nonempty. Since the fixed-point set of a nonexpansive map on a closed convex subset of a Hilbert space is closed and convex, $S$ is closed and convex. Hence the metric projection $\operatorname{proj}_S(a)$ is well-defined and unique.

We use the standard projection characterization: for a nonempty closed convex set $S$ and $q=\operatorname{proj}_S(a)$,
\begin{equation}
    \langle a-q,y-q\rangle\le 0,
    \qquad
    \forall y\in S .
    \label{eq:app-projection-characterization}
\end{equation}
All convergence is in the Euclidean norm of the latent space. The full inverse trajectory is obtained by applying the same per-step argument sequentially for $j=1,\ldots,T$.

For convenience, we also record the two standard convergence theorems used in the proofs below.

\begin{theorem}[Halpern iteration; Bauschke 1996, Sec.~3, Thm.~3.1, single-map case]
\label{thm:app-halpern-bauschke}
Let $H$ be a real Hilbert space, $C\subset H$ a nonempty closed convex set, and $T:C\to C$ a nonexpansive map with $\operatorname{Fix}(T)\neq\emptyset$. Let $a,x^0\in C$, and let $(\lambda_k)\subset[0,1]$ satisfy
\begin{equation}
    \lambda_k\to 0,
    \qquad
    \sum_{k=1}^{\infty}\lambda_k=\infty,
    \qquad
    \sum_{k=1}^{\infty}|\lambda_{k+1}-\lambda_k|<\infty .
    \label{eq:app-halpern-conditions-bauschke}
\end{equation}
Then the sequence defined by
\begin{equation}
    x^{k+1}=\lambda_{k+1}a+(1-\lambda_{k+1})T(x^k)
    \label{eq:app-halpern-bauschke-update}
\end{equation}
converges strongly to $\operatorname{proj}_{\operatorname{Fix}(T)}(a)$~\cite{bauschke1996approximation}.
\end{theorem}

\begin{theorem}[Modified Halpern iteration; Cuntavepanit--Panyanak 2011, Sec.~3, Thm.~3.2]
\label{thm:app-modified-halpern-cp}
Let $C$ be a nonempty closed convex subset of a complete CAT(0) space, let $T:C\to C$ be nonexpansive with $F(T)\neq\emptyset$, and let $u\in C$, $x^1\in C$. Let $(\lambda_n),(\tau_n)\subset(0,1)$ satisfy
\begin{equation}
    \lambda_n\to 0,
    \qquad
    \sum_{n=1}^{\infty}\lambda_n=\infty,
    \qquad
    0<\liminf_{n\to\infty}\tau_n\le\limsup_{n\to\infty}\tau_n<1 .
    \label{eq:app-modified-halpern-conditions}
\end{equation}
Then the sequence defined by
\begin{equation}
    x^{n+1}=\tau_n x^n\oplus(1-\tau_n)\bigl(\lambda_n u\oplus(1-\lambda_n)Tx^n\bigr)
    \label{eq:app-modified-halpern-update}
\end{equation}
converges to the point of $F(T)$ nearest to $u$~\cite{cuntavepanit2011strong}. In a Euclidean or Hilbert space, $\oplus$ denotes the usual convex combination, $F(T)=\operatorname{Fix}(T)$, and the nearest point is $\operatorname{proj}_{\operatorname{Fix}(T)}(u)$.
\end{theorem}

\subsection{Algebra of the Straightness Selector}
\label{app:straightness-selector-algebra}

This subsection expands the straightness selector used in Sec.~\ref{sec:method}. The window size $W$ denotes the number of previous velocities used to estimate the local mean velocity. For steps where the full window is available,
\begin{equation}
    \bar v^W_{j-1}
    :=
    \frac{\sum_{i=1}^{W} h_{j-i}v_{j-i}}
         {\sum_{i=1}^{W} h_{j-i}},
    \qquad
    v_{j-i}:=\frac{x_{j-i}-x_{j-i-1}}{h_{j-i}} .
    \label{eq:app-windowed-past-velocity}
\end{equation}
For early inversion steps where not enough previous velocities are available, the implementation includes as many velocities as possible from the stored trajectories.

For a candidate current state $x$ at time $t_j$, the candidate velocity is
\begin{equation}
    v_j(x):=\frac{x-x_{j-1}}{h_j}.
\end{equation}
The causal local straightness selector is
\begin{equation}
    \phi^{\mathrm{str}}_{j,W}(x)
    :=
    h_j\left\|v_j(x)-\bar v^W_{j-1}\right\|_2^2 .
    \label{eq:app-selector-definition}
\end{equation}
Substituting $v_j(x)$ gives
\begin{align}
    \phi^{\mathrm{str}}_{j,W}(x)
    &=
    h_j
    \left\|
        \frac{x-x_{j-1}}{h_j}
        -
        \bar v^W_{j-1}
    \right\|_2^2
    \label{eq:app-selector-algebra-1}
    \\
    &=
    \frac{1}{h_j}
    \left\|
        x-
        \left(x_{j-1}+h_j\bar v^W_{j-1}\right)
    \right\|_2^2 .
    \label{eq:app-selector-algebra-2}
\end{align}
Thus the straightness anchor is
\begin{equation}
    a_{j,W}
    :=
    x_{j-1}+h_j\bar v^W_{j-1},
    \label{eq:app-straightness-anchor}
\end{equation}
and the selector is exactly
\begin{equation}
    \phi^{\mathrm{str}}_{j,W}(x)
    =
    \frac{1}{h_j}\left\|x-a_{j,W}\right\|_2^2 .
    \label{eq:app-selector-quadratic}
\end{equation}
Since $h_j>0$, minimizing $\phi^{\mathrm{str}}_{j,W}$ over any set is equivalent to minimizing the squared distance to $a_{j,W}$ over the same set. In particular,
\begin{equation}
    \arg\min_{x\in S_j^C}\phi^{\mathrm{str}}_{j,W}(x)
    =
    \arg\min_{x\in S_j^C}\left\|x-a_{j,W}\right\|_2^2
    =
    \left\{\operatorname{proj}_{S_j^C}(a_{j,W})\right\}.
    \label{eq:app-selector-projection}
\end{equation}
This identity is the algebraic link between causal straightness minimization and projection-selected fixed-point inversion.

\subsection{Verification of the Implemented Anchor Schedule}
\label{app:alpha-schedule-verification}

In this work, we use the following anchor schedule, identical to Eq.~\eqref{eqn:alpha-schedule}:
\begin{equation}
    \alpha_k=\alpha_1\frac{\delta}{(k-1)+\delta},
    \qquad k\ge 1,\ \ 0<\alpha_1<1,\ \ \delta>0.
    \label{eq:app-alpha-schedule}
\end{equation}
Below we show the schedule satisfies the Halpern conditions in Eq.~\eqref{eq:app-halpern-conditions}.
We can easily see $\alpha_k\in(0,1)$ and $\alpha_k\to0$. Moreover,
\begin{equation}
    \sum_{k=1}^{\infty}\alpha_k
    =
    \alpha_1\delta\sum_{k=1}^{\infty}\frac{1}{(k-1)+\delta}
    =
    \infty,
\end{equation}
because the sum is harmonic up to a finite shift. Finally, $(\alpha_k)$ is monotone decreasing, so
\begin{equation}
    \sum_{k=2}^{\infty}|\alpha_k-\alpha_{k-1}|
    =
    \alpha_1-\lim_{k\to\infty}\alpha_k
    =
    \alpha_1
    <
    \infty .
\end{equation}

\subsection{Proof of Theorem~\ref{thm:straightness_anchored_fpi}}
\label{app:anchored_fpi_proof}

\begin{theorem*}[Restatement of Theorem~\ref{thm:straightness_anchored_fpi}]
Under Assumption~\ref{assump:local_root_selection} and with $\alpha_k$ satisfying Eq.~\eqref{eq:vanishing_anchor_schedule}, the anchored fixed-point iteration
\begin{equation*}
    x_j^k=\alpha_k a_{j,W}+(1-\alpha_k)P_j(x_j^{k-1})
\end{equation*}
converges to
\begin{equation*}
    x_j^\star
    \in
    \arg\min_{x\in S_j^C}
    \phi^{\mathrm{str}}_{j,W}(x).
\end{equation*}
\end{theorem*}

\begin{proof}
Fix an inversion step $j$ and use the shorthand in Eq.~\eqref{eq:app-proof-shorthand}. Under Assumption~\ref{assump:local_root_selection}, $C$ is closed and convex, $P(C)\subset C$, $P$ is nonexpansive on $C$, $a,x^0\in C$, and $S=\operatorname{Fix}(P)\cap C$ is nonempty. Hence $\operatorname{proj}_S(a)$ is well-defined and unique.

The anchored fixed-point iteration in the main text can be written as
\begin{equation}
    x^k
    =
    \alpha_k a+(1-\alpha_k)P(x^{k-1}),
    \qquad
    k\ge 1,
    \label{eq:app-halpern-local}
\end{equation}
where the schedule $(\alpha_k)$ satisfies
\begin{equation}
    \alpha_k\in(0,1),
    \qquad
    \alpha_k\to0,
    \qquad
    \sum_{k=1}^{\infty}\alpha_k=\infty,
    \qquad
    \sum_{k=2}^{\infty}|\alpha_k-\alpha_{k-1}|<\infty
    \label{eq:app-halpern-conditions}
\end{equation}
by Eq.~\eqref{eq:vanishing_anchor_schedule}. Applying Theorem~\ref{thm:app-halpern-bauschke} with $T=P$, anchor $a=a_{j,W}$, and $\lambda_k=\alpha_k$, we obtain
\begin{equation}
    x^k\to q:=\operatorname{proj}_S(a).
    \label{eq:app-halpern-limit}
\end{equation}
By Eq.~\eqref{eq:app-selector-quadratic},
\begin{equation}
    \phi^{\mathrm{str}}_{j,W}(x)
    =
    \frac{1}{h_j}\left\|x-a_{j,W}\right\|_2^2 ,
\end{equation}
so minimizing $\phi^{\mathrm{str}}_{j,W}$ over $S_j^C$ is identical to projecting $a_{j,W}$ onto $S_j^C$. Hence
\begin{equation}
    q
    =
    \operatorname{proj}_{S_j^C}(a_{j,W})
    \in
    \arg\min_{x\in S_j^C}
    \phi^{\mathrm{str}}_{j,W}(x).
\end{equation}
The limit is therefore an exact local inverse root selected by the causal straightness criterion. This proves Theorem~\ref{thm:straightness_anchored_fpi}.
\end{proof}
\subsection{Proof of Theorem~\ref{thm:momentum_preserves_selection}}
\label{app:momentum_proof}

\begin{theorem*}[Restatement of Theorem~\ref{thm:momentum_preserves_selection}]
Under the same assumptions as Theorem~\ref{thm:straightness_anchored_fpi}, with $\alpha_k$ in Eq.~\eqref{eqn:alpha-schedule} and momentum weight $\mu\in[0,1)$, the decoupled-momentum iteration
\begin{align*}
    M_j^k &= \mu M_j^{k-1}+(1-\mu)P_j(x_j^{k-1}), \\
    x_j^k &= \alpha_k a_{j,W}+(1-\alpha_k)M_j^k,
\end{align*}
initialized with $M_j^0=x_j^0$, satisfies
\begin{equation*}
    x_j^k\to x_j^\star,
    \qquad
    x_j^\star\in\arg\min_{x\in S_j^C}\phi^{\mathrm{str}}_{j,W}(x).
\end{equation*}
\end{theorem*}

\begin{proof}
{}The strategy is to eliminate the momentum variable for $k\ge1$ using the anchored relation $x^k=\alpha_k a+(1-\alpha_k)M^k$, recast the resulting one-state recursion in the modified-Halpern form of Theorem~\ref{thm:app-modified-halpern-cp}, and verify its coefficient hypotheses. The case $\mu=0$ reduces to Theorem~\ref{thm:straightness_anchored_fpi}, so we focus on $0<\mu<1$.

Fix $j$ and write $P:=P_j$, $C:=C_j$, $S:=S_j^C$, $a:=a_{j,W}$. The implemented recursion is
\begin{equation}
    M^{k+1}=\mu M^k+(1-\mu)P(x^k),
    \qquad
    x^k=\alpha_k a+(1-\alpha_k)M^k,
    \qquad
    M^0=x^0\in C,
    \label{eq:app-momentum-recursion}
\end{equation}
with $\mu\in[0,1)$. Since $C$ is convex, $a\in C$, and $P(C)\subset C$, induction gives $M^k,x^k\in C$ for all $k$.

For $k\ge 1$, the anchored relation inverts to
\begin{equation}
    M^k=\frac{x^k-\alpha_k a}{1-\alpha_k},
    \qquad k\ge 1.
    \label{eq:app-Mk-from-xk}
\end{equation}
Substituting Eq.~\eqref{eq:app-Mk-from-xk} into the momentum and anchor updates gives
\begin{equation}
    x^{k+1}=\eta_k a+\rho_k x^k+\sigma_k P(x^k),
    \qquad k\ge 1,
    \label{eq:app-onestate}
\end{equation}
with
\begin{equation}
    \rho_k=\frac{\mu(1-\alpha_{k+1})}{1-\alpha_k},
    \quad
    \sigma_k=(1-\mu)(1-\alpha_{k+1}),
    \quad
    \eta_k=\alpha_{k+1}-\frac{\mu(1-\alpha_{k+1})\alpha_k}{1-\alpha_k},
\end{equation}
summing to $\eta_k+\rho_k+\sigma_k=1$. Setting $b_k:=\rho_k$ and $\gamma_k:=\eta_k/(1-\rho_k)$, this regroups as
\begin{equation}
    x^{k+1}=b_k\,x^k+(1-b_k)\bigl(\gamma_k a+(1-\gamma_k)P(x^k)\bigr),
    \label{eq:app-modified-halpern-form}
\end{equation}
which matches Theorem~\ref{thm:app-modified-halpern-cp} with $\tau_n\!\leftarrow\!b_k$, $\lambda_n\!\leftarrow\!\gamma_k$, anchor $u=a$, and map $T=P$ in Euclidean space.

It remains to verify that $b_k,\gamma_k\in(0,1)$ for all sufficiently large $k$ and that they satisfy the hypotheses of Theorem~\ref{thm:app-modified-halpern-cp}. The explicit schedule $\alpha_k=\alpha_1\delta/((k-1)+\delta)$ with $0<\alpha_1<1$ and $\delta>0$ has three structural properties that make this verification immediate:
\begin{enumerate}[label=(\roman*),leftmargin=*]
    \item $\alpha_k$ is strictly decreasing with $0<\alpha_k\le\alpha_1<1$ for all $k\ge 1$, and thus $1-\alpha_k\ge 1-\alpha_1>0$;
    \item $\alpha_k\to 0$;
    \item $\alpha_{k+1}/\alpha_k=(k-1+\delta)/(k+\delta)\to 1$.
\end{enumerate}
Combining (ii) and (iii),
\begin{equation}
    b_k
    =
    \mu\,\frac{1-\alpha_{k+1}}{1-\alpha_k}
    \;\longrightarrow\;\mu,
    \qquad
    \frac{\eta_k}{\alpha_k}
    =
    \frac{\alpha_{k+1}}{\alpha_k}-\mu\,\frac{1-\alpha_{k+1}}{1-\alpha_k}
    \;\longrightarrow\;1-\mu .
\end{equation}
Both limits lie strictly inside $(0,1)$ because $\mu\in(0,1)$, so $b_k\in(0,1)$ and $\eta_k>0$ hold for all sufficiently large $k$.
Together with $\sigma_k=(1-\mu)(1-\alpha_{k+1})>0$ (always) and $1-\rho_k=\eta_k+\sigma_k$, this also gives $\gamma_k=\eta_k/(1-\rho_k)\in(0,1)$ eventually.
Discarding the finite prefix we have $b_k,\gamma_k\in(0,1)$ for all sufficiently large $k$.
Moreover,
\begin{equation}
    \frac{\gamma_k}{\alpha_k}
    =
    \frac{\eta_k/\alpha_k}{1-\rho_k}
    \;\longrightarrow\;\frac{1-\mu}{1-\mu}=1,
\end{equation}
so $\gamma_k\sim\alpha_k$. Therefore $0<\liminf_k b_k=\limsup_k b_k=\mu<1$, $\gamma_k\to 0$, and $\sum_k\gamma_k=\infty$ since $\sum_k\alpha_k=\infty$ (App.~\ref{app:alpha-schedule-verification}).

Theorem~\ref{thm:app-modified-halpern-cp} applied to Eq.~\eqref{eq:app-modified-halpern-form} therefore gives $x^k\to q:=\operatorname{proj}_S(a)$, and Eq.~\eqref{eq:app-Mk-from-xk} together with $\alpha_k\to 0$ gives $M^k\to q$ as well. The quadratic selector identity Eq.~\eqref{eq:app-selector-quadratic} then yields
\begin{equation}
    q=\operatorname{proj}_{S_j^C}(a_{j,W})\in\arg\min_{x\in S_j^C}\phi^{\mathrm{str}}_{j,W}(x),
\end{equation}
completing the proof.
\end{proof}

\subsection{Finite-Distance Bound for Approximate Iterates}
\label{app:approximate-iterate-bound}

While our convergence guarantees are asymptotic, in practice the fixed-point iteration runs for only finitely many steps, and the recovered iterate retains a nonzero residual. The following remark shows that, even in this finite regime, the distance from the iterate to the selected fixed point controls both the residual and the gap to the straightness optimum.

\begin{remark}[Finite-distance bound for approximate iterates]
Let $q_j$ denote the selected fixed point, and let $x_j^k$ be a finite inner iterate lying in the local region where $P_j$ is nonexpansive. Define
\begin{equation}
    d_k := \|x_j^k-q_j\|,
    \qquad
    R_j := \|q_j-a_{j,W}\|.
\end{equation}
Since $q_j\in\operatorname{Fix}(P_j)$, the residual satisfies
\begin{equation}
    \|r_j(x_j^k)\|
    =
    \|P_j(x_j^k)-x_j^k\|
    \le
    \|P_j(x_j^k)-P_j(q_j)\|+\|x_j^k-q_j\|
    \le
    2d_k .
\end{equation}
Moreover, for the straightness selector
\begin{equation}
    \phi^{\mathrm{str}}_{j,W}(x)=\frac{1}{h_j}\|x-a_{j,W}\|^2,
\end{equation}
we have
\begin{equation}
    \phi^{\mathrm{str}}_{j,W}(x_j^k)
    -
    \phi^{\mathrm{str}}_{j,W}(q_j)
    =
    \frac{1}{h_j}
    \left(
    2\langle q_j-a_{j,W},\,x_j^k-q_j\rangle
    +
    \|x_j^k-q_j\|^2
    \right),
\end{equation}
and therefore
\begin{equation}
    \phi^{\mathrm{str}}_{j,W}(x_j^k)
    \le
    \min_{s\in S_j^C}\phi^{\mathrm{str}}_{j,W}(s)
    +
    \frac{2R_j d_k+d_k^2}{h_j}.
\end{equation}
In particular, if $d_k\le \varepsilon$, then
\begin{equation}
    \|r_j(x_j^k)\|\le 2\varepsilon,
    \qquad
    \phi^{\mathrm{str}}_{j,W}(x_j^k)
    \le
    \min_{s\in S_j^C}\phi^{\mathrm{str}}_{j,W}(s)
    +
    \frac{2R_j\varepsilon+\varepsilon^2}{h_j}.
\end{equation}
Hence finite iterates sufficiently close to the selected fixed point are both approximately feasible and near-optimal in straightness.
\end{remark}

\subsection{Finite-iteration Role of Decoupled Momentum}
\label{app:decoupled_momentum}

We first provide an intuitive local explanation of how decoupled momentum can suppress oscillatory modes in the finite-iteration regime. The key intuition is that momentum retains a component of the current iterate, which can partially cancel a fixed-point proposal pointing in the opposite direction.

For one inversion step, let $P$ be the local fixed-point map, $a$ the fixed straightness anchor, and $x^k$ the $k$-th iterate. Eliminating the momentum variable from Eqs.~\eqref{eq:decoupled_momentum_state}--\eqref{eq:decoupled_momentum_anchor} gives
\begin{equation}
x^{k+1} = \eta_k a+\rho_k x^k+\sigma_kP(x^k),
\end{equation}
where $\rho_k =\frac{\mu(1-\alpha_{k+1})}{1-\alpha_k}, \sigma_k =(1-\mu)(1-\alpha_{k+1}), \eta_k=1-\rho_k-\sigma_k$.

Let $q$ be a fixed point satisfying $P(q)=q$, and define the current error and fixed-point proposal relative to $q$ as $e_k:=x^k-q$ and $p_k:=P(x^k)-q$. Then,
\begin{equation}
e_{k+1} =\eta_k(a-q)+\rho_ke_k+\sigma_kp_k,
\end{equation}
where the term $\rho_ke_k$ is present only when momentum is used ($\mu>0$). Therefore, when $p_k$ and $e_k$ point in opposing directions, $\rho_ke_k$ can partially cancel the $\sigma_kp_k$ term, reducing the oscillatory error-dependent component.

To illustrate this mechanism, consider a local oscillatory case in which $p_k=-e_k$. Without momentum, the error-dependent component becomes
\begin{equation}
    \rho_ke_k+\sigma_kp_k = -(1-\alpha_{k+1})e_k.
\end{equation}
With $\mu=0.5$, it becomes
\begin{equation}
    \rho_ke_k+\sigma_kp_k = (1-\alpha_{k+1})\frac{\alpha_k}{2(1-\alpha_k)}e_k.
\end{equation}
Thus, momentum removes the sign reversal of this component and reduces its magnitude to a fraction
\begin{equation}
    \frac{\alpha_k}{2(1-\alpha_k)} \le \frac{1}{2}
\end{equation}
of the no-momentum magnitude, since our schedule satisfies $\alpha_k\le1/2$.
This illustrates how momentum can suppress sign-alternating error propagation near a fixed point.

\paragraph{Numerical Results}
On top of the intuitive explanation, we empirically examine whether the predicted reduction in oscillatory behavior appears in the actual inner-loop iterates.
Let $d_k:=x^{k+1}-x^k$, and define the cosine similarity between consecutive movements as
\begin{equation}
c_k := \frac{\langle d_k,d_{k-1}\rangle} {\Vert d_k\Vert_2\Vert d_{k-1}\Vert_2}.
\end{equation}
Here, $c_k<0$ indicates that two consecutive updates point in opposing directions.

\begin{table}[t]
\centering
\footnotesize
\caption{Oscillation analysis for non-momentum and momentum updates.}
\label{tab:oscillation}
\begin{tabular}{llcc}
\toprule
\textbf{Metric} & \textbf{Definition} & $\mu=0$ & $\mu=0.5$ \\
\midrule
Update flip rate & $\mathbb{E}[\mathbf{1}(c_k<0)]$ & 0.968 & 0.336 (-65.3\%) \\
Negative cosine magnitude & $\mathbb{E}[\max(0,-c_k)]$ & 0.699 & 0.082 (-88.3\%) \\
Last-iteration residual RMS & $\mathrm{RMS}(P(x^{K-1})-x^{K-1})$ & 0.00792 & 0.00370 (-53.2\%) \\
\bottomrule
\end{tabular}
\end{table}

As shown in Tab.~\ref{tab:oscillation}, decoupled momentum substantially reduces both the frequency and aggregate magnitude of back-and-forth updates, accompanied by a 53.2\% lower finite-iteration residual. These results are consistent with the local oscillation-damping mechanism described above.

\section{Experimental Details}
\label{app:exp-details}

\begin{table}[t]
\centering
\small
\caption{Default hyperparameters for the main reconstruction experiments.}
\label{tab:app-default-hparams}
\begin{tabular}{ll}
\toprule
\textbf{Quantity} & \textbf{Value} \\
\midrule
Backbone & FLUX.1-dev \\
Number of timesteps $T$ & 15 \\
Base solver & Euler \\
Fixed-point iterations $K$ & 10 \\
Previous-velocity window size $W$ & 1 \\
Momentum coefficient $\mu$ & 0.5 \\
Initial anchor weight $\alpha_1$ & 0.5 \\

Decay parameter $\delta$ & $1/8$ \\
Guidance scale & 1 \\
\bottomrule
\end{tabular}
\end{table}

\begin{table}[t]
\centering
\small
\caption{Default hyperparameters for the main editing experiments.}
\label{tab:app-default-hparams-edit}
\begin{tabular}{ll}
\toprule
\textbf{Quantity} & \textbf{Value} \\
\midrule
Backbone & FLUX.1-dev \\
Number of timesteps $T$ & 15
\\
Base solver & Euler \\
Fixed-point iterations $K$ & 4 \\
Previous-velocity window size $W$ & 1 \\
Momentum coefficient $\mu$ & 0.1 \\
Initial anchor weight $\alpha_1$ & 0.5 \\

Decay parameter $\delta$ & 2 \\
Guidance scale & 2 \\
\bottomrule
\end{tabular}
\end{table}

The default hyperparameters used in the main reconstruction and editing experiments can be found in Tab.~\ref{tab:app-default-hparams} and Tab.~\ref{tab:app-default-hparams-edit}, respectively.

\paragraph{Baselines}
\label{app:baseline-details}
\begin{table}[!t]
\centering
\footnotesize
\setlength{\tabcolsep}{4pt}
\caption{NFE budget used for inversion and generation. For fixed-point inversion methods, the inversion cost is computed as number of inversion steps times number of fixed-point iterations.}
\label{tab:nfe-budget}
\begin{tabular}{@{}llcccc@{}}
\toprule
\textbf{Experiment}
& \textbf{Method}
& \textbf{Inversion}
& \textbf{Generation}
& \textbf{Extra}
& \textbf{Total NFE} \\
\midrule

\multirow{4}{*}{Reconstruction}
& Fixed-point inversions
& $15 \times 10 = 150$
& $15$
& $0$
& $165$ \\
& ReFlow
& $83$
& $83$
& $0$
& $166$ \\
& RF-Solver
& $42 \times 2 = 84$
& $42 \times 2 = 84$
& $0$
& $168$ \\
& FireFlow
& $83$
& $83$
& $2$
& $168$ \\
\midrule

\multirow{4}{*}{Editing}
& Fixed-point inversions
& $15 \times 4 = 60$
& $15$
& $0$
& $75$ \\
& ReFlow
& $38$
& $38$
& $0$
& $76$ \\
& RF-Solver
& $19 \times 2 = 38$
& $19 \times 2 = 38$
& $0$
& $76$ \\
& FireFlow
& $38$
& $38$
& $2$
& $78$ \\

\bottomrule
\end{tabular}
\end{table}

In our experiments, we compare SelFix with ReFlow, RF-Solver, FireFlow, ReNoise, AIDI, and naive FPI.
For fair comparison, we approximately match the total NFE across inversion methods by increasing the number of denoising steps for ReFlow, RF-Solver, and FireFlow, as shown in Tab.~\ref{tab:nfe-budget}.
In reconstruction, fixed-point inversion methods, including SelFix, use $165$ NFEs: $150$ for inversion with $15$ steps and $10$ fixed-point iterations, followed by $15$ NFEs for generation.
We therefore use $83$ steps per direction for ReFlow, $42$ steps per direction for RF-Solver, and $83$ steps per direction for FireFlow, resulting in $166$, $168$, and $168$ total NFEs, respectively.
In editing, fixed-point inversion methods use $75$ NFEs: $60$ for inversion with $15$ steps and $4$ iterations, followed by $15$ NFEs for generation.
We accordingly use $38$ steps per direction for ReFlow and FireFlow, and $19$ steps per direction for RF-Solver, resulting in $76$, $78$, and $76$ total NFEs, respectively.

To implement ReNoise on the Flux backbone, we try various settings for the accumulation weighting scheme and the editability enhancement step, and report the optimal setting.
Specifically, we compare the weighting strategy proposed in the original paper against a uniform weighting scheme.
For the editability enhancement step, we try the hyperparameters of the original paper ($\lambda_{pair}$: 10, $\lambda_{patch-KL}$: 0.055), but obtain the best result by disabling it.
For the editing experiment, which uses a 4-iteration setting, the original paper provides no predefined configuration, so we adopt the optimal setting from the reconstruction experiment.
When implementing AIDI, we use the suggested \texttt{AIDI\_E} variant.

\paragraph{Editing Pipeline}
\label{app:editing-pipeline-details}

\begin{table}[!t]
\centering
\caption{Hyperparameter settings for the editing task.}
\begin{tabular}{l|c}
\toprule
\textbf{Argument} & \textbf{Value} \\
\midrule
\texttt{--start\_layer\_index} & 0 \\
\texttt{--end\_layer\_index}   & 37 \\
\texttt{--inject}            & 1 \\
\texttt{--reuse\_v}          & 1 \\
\texttt{--editing\_strategy} & \texttt{replace\_v} \\
\texttt{--qkv\_ratio}        & \texttt{1.0,1.0,1.0} \\
\bottomrule
\end{tabular}
\label{tab:app-feature-injection}
\end{table}

For all inversion methods, the downstream editing procedure is kept fixed so that the comparison isolates the inversion method.
Our implementation uses value feature injection: the inverted source trajectory is used as the source trajectory, and the target prompt is used during the editing pass.
Feature injection settings are shared across all inversion baselines.
We use the official repository of FireFlow~\cite{deng2025fireflow} to implement the feature injection, with hyperparameters set as in Tab.~\ref{tab:app-feature-injection}.

\subsection{Compute and Runtime}
\label{app:compute-runtime-details}

All experiments are conducted in a GPU cloud environment, with each experiment running on a single H200 GPU.
When using SelFix for reconstruction and editing tasks, processing a single image takes roughly 11 seconds and 5.5 seconds, respectively.

\section{Additional Experimental Results}
\label{app:additional-results}

\subsection{Robustness to T2I Backbone and Dataset}
\label{app:Robustness}
To examine the generality of SelFix beyond FLUX.1-dev and PIE-Bench, we evaluate it with a different backbone (Stable Diffusion 3.5 Large) and on a different dataset (COCO~\cite{lin2014mscoco}).
For SD3.5 Large, we use the same hyperparameters as the FLUX.1-dev experiments unless otherwise noted. In reconstruction, all hyperparameters follow Tab.~\ref{tab:app-default-hparams}.
In editing, we follow Tab.~\ref{tab:app-default-hparams-edit}, except that we use $\mu{=}0.3$ and $\delta{=}1$, and a guidance scale of 1 for inversion and 3.5 for generation.
The NFE budget is matched across methods as in Tab.~\ref{tab:nfe-budget}.
The results are reported in Tab.~\ref{tab:generality}.

For the COCO dataset, we follow the convention of ReNoise~\cite{garibi2024renoise}, using the \texttt{val2017} split images whose width and height both exceed 420 pixels ($N{=}3866$), paired with the first caption of each image as the prompt.
Images are converted to RGB and resized to $512{\times}512$ with bicubic interpolation without preserving the aspect ratio.
We use the same hyperparameters as the FLUX.1-dev reconstruction setting in Tab.~\ref{tab:app-default-hparams}.

\begin{table}[t]
\centering
\caption{Generality of SelFix across backbones and datasets. All methods are compared under matched NFE, using the same baselines as in the main experiments.}
\label{tab:generality}
\small
\setlength{\tabcolsep}{3.5pt}
\begin{subtable}[t]{0.32\linewidth}
  \centering
  \caption{PIE-Bench recon. (SD3.5)}
  \label{tab:sd35_recon}
  \begin{tabular}{lcc}
  \toprule
  Method & LPIPS $\downarrow$ & PSNR $\uparrow$ \\
  \midrule
  ReFlow    & 0.3071 & 18.60 \\
  RF-Solver & 0.1426 & 23.80 \\
  FireFlow  & 0.0910 & 27.06 \\
  ReNoise   & 0.1335 & 23.56 \\
  AIDI      & 0.2040 & 23.12 \\
  FPI       & 0.3253 & 14.52 \\
  SelFix    & \textbf{0.0561} & \textbf{29.42} \\
  \bottomrule
  \end{tabular}
\end{subtable}\hfill
\begin{subtable}[t]{0.32\linewidth}
  \centering
  \caption{PIE-Bench edit. (SD3.5)}
  \label{tab:sd35_edit}
  \begin{tabular}{lcc}
  \toprule
  Method & CLIP $\uparrow$ & PSNR $\uparrow$ \\
  \midrule
  ReFlow    & 28.21 & 14.75 \\
  RF-Solver & 28.22 & 16.13 \\
  FireFlow  & 28.18 & 15.04 \\
  ReNoise   & 28.26 & 15.83 \\
  AIDI      & 28.23 & 16.90 \\
  FPI       & 27.96 & 16.46 \\
  SelFix    & \textbf{28.34} & \textbf{17.21} \\
  \bottomrule
  \end{tabular}
\end{subtable}\hfill
\begin{subtable}[t]{0.32\linewidth}
  \centering
  \caption{COCO recon. (FLUX.1-dev)}
  \label{tab:coco_recon}
  \begin{tabular}{lcc}
  \toprule
  Method & LPIPS $\downarrow$ & PSNR $\uparrow$ \\
  \midrule
  ReFlow    & 0.2748 & 18.68 \\
  RF-Solver & 0.2556 & 19.49 \\
  FireFlow  & 0.1523 & 23.75 \\
  ReNoise   & 0.1337 & 21.92 \\
  AIDI      & 0.1590 & 22.45 \\
  FPI       & 0.3588 & 13.76 \\
  SelFix    & \textbf{0.0913} & \textbf{26.15} \\
  \bottomrule
  \end{tabular}
\end{subtable}
\end{table}

\subsection{Comparison with Additional Baselines}
\label{app:additional-baselines}
We provide additional comparisons with UniEdit-Flow~\cite{jiao2026unieditflow} and PMI~\cite{wang2026freelunch}.
We follow the same evaluation protocol as the main reconstruction experiment, with approximately matched NFE budgets for fair comparison.
Since PMI is an auxiliary correction applied to a base inversion method, we evaluate it with three representative methods: ReFlow, FireFlow, and FPI, corresponding to naive reverse integration, an improved solver-based method, and fixed-point inversion, respectively.
As shown in Tab.~\ref{tab:additional-baselines}, SelFix outperforms all additional baselines in both LPIPS and PSNR.

\begin{table}[t]
\centering
\caption{Comparison with additional inversion baselines on PIE-Bench reconstruction using FLUX.1-dev.}
\label{tab:additional-baselines}
\small
\begin{tabular}{lcc}
\toprule
\textbf{Method} & \textbf{LPIPS} $\downarrow$ & \textbf{PSNR} $\uparrow$ \\
\midrule
UniEdit-Flow   & 0.0971 & 27.17 \\
ReFlow + PMI   & 0.2388 & 19.53 \\
FireFlow + PMI & 0.1027 & 27.05 \\
FPI + PMI      & 0.3068 & 15.04 \\
\midrule
\textbf{SelFix (Ours)} & \textbf{0.0742} & \textbf{28.29} \\
\bottomrule
\end{tabular}
\end{table}

\subsection{Preservation-Alignment Trade-off}
\label{app:tradeoff}

SelFix improves the inversion process rather than the downstream editing mechanism.
In the main experiments, we use the same editing strategy across all inversion methods to isolate the effect of inversion.
More accurate inversion provides a more faithful source trajectory, which primarily benefits source preservation.
In contrast, target-prompt alignment depends largely on the downstream editing strategy.
Since this strategy is fixed across methods, we expect SelFix to improve source preservation while maintaining comparable target-prompt alignment.

Source preservation and target-prompt alignment exhibit a trade-off~\cite{deng2025fireflow}, as stronger preservation can restrict the changes needed to follow the target prompt.
As shown in Fig.~\ref{fig:pareto-edit}, SelFix achieves the best source preservation while maintaining comparable target-prompt alignment among the compared methods.
To further examine this trade-off, we evaluate SelFix with two additional feature-injection strategies from FireFlow~\cite{deng2025fireflow}, Add~Q and Add~K, while keeping the inversion process unchanged.
The results are reported in Tab.~\ref{tab:edit_strategy}.

\begin{table}[t]
\centering
\caption{SelFix under different editing strategies on PIE-Bench with FLUX.1-dev. Only the downstream feature-injection strategy is varied; the SelFix inversion remains unchanged.}
\label{tab:edit_strategy}
\small
\begin{tabular}{lcc}
\toprule
Editing strategy & Preservation PSNR $\uparrow$ & CLIP Sim. (Whole) $\uparrow$ \\
\midrule
Add Q                 & 15.34 & \textbf{27.71} \\
Add K                 & 15.62 & 27.59 \\
Replace V (Default)   & \textbf{19.79} & 27.29 \\
\bottomrule
\end{tabular}
\end{table}

Consistent with FireFlow, Add~Q and Add~K improve target-prompt alignment compared to the default Replace~V strategy, but substantially reduce source preservation.
Specifically, Add~Q increases whole-image CLIP similarity from 27.29 to 27.71 while reducing background PSNR from 19.79 to 15.34.
These results illustrate that the preservation-alignment trade-off can be adjusted through the downstream editing strategy without changing SelFix inversion.

\subsection{Additional Reconstruction Qualitative Results}
\label{app:additional-reconstruction}
We provide more reconstruction examples from PIE-Bench, expanding Fig.~\ref{fig:qual-recon}.
The results are in Fig.~\ref{fig:app-additional-recon-grid}.

\subsection{Additional Editing Qualitative Results}
\label{app:additional-editing}

We provide more prompt-based editing examples from PIE-Bench, expanding Fig.~\ref{fig:qual-edit}.
The results are in Fig.~\ref{fig:app-additional-editing-grid}.

\begin{figure}[p]
    \centering
    \includegraphics[height=0.95\textheight, keepaspectratio]{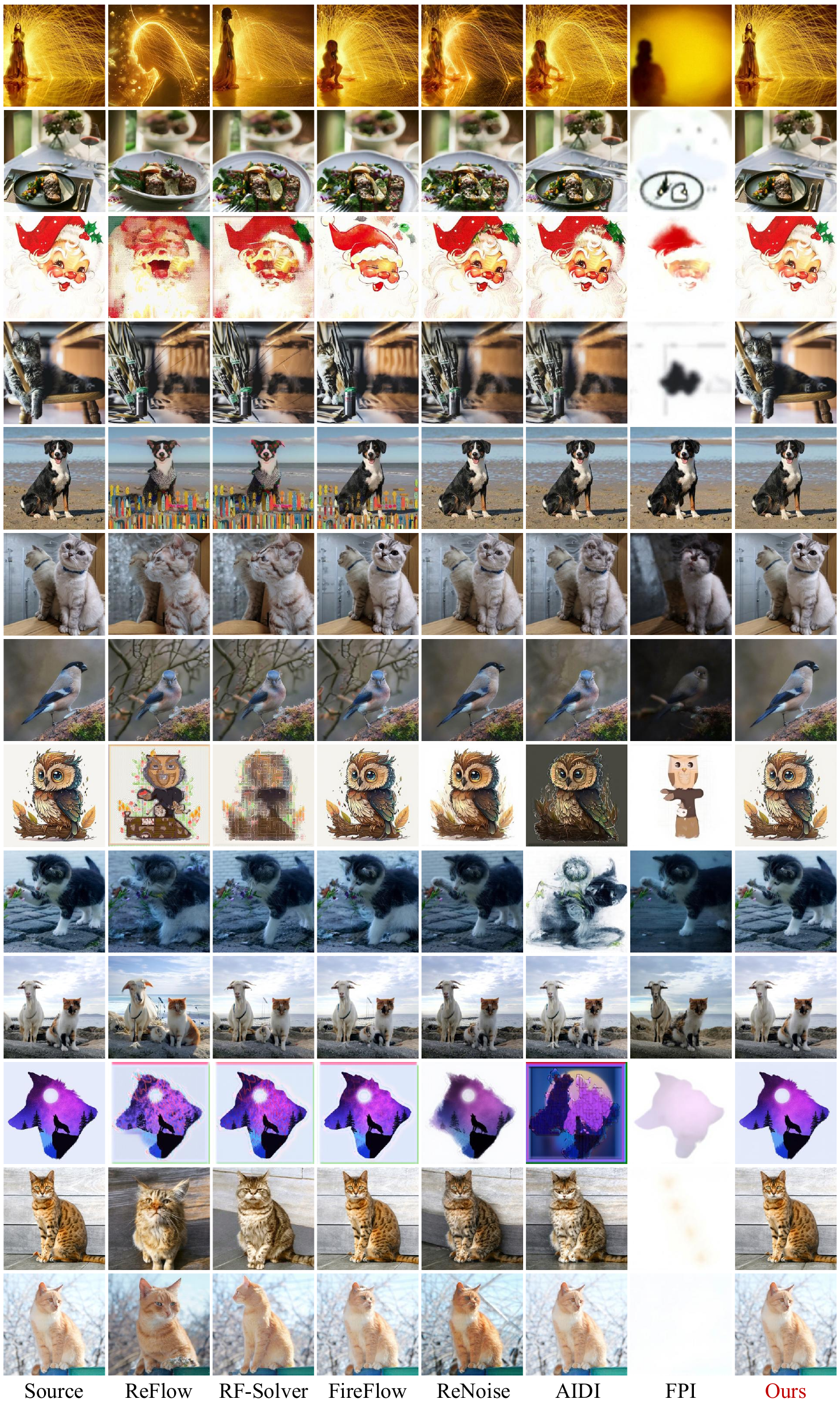}
    \caption{Additional qualitative results for the image reconstruction task.}
    \label{fig:app-additional-recon-grid}
\end{figure}

\begin{figure}[p]
    \centering
    \includegraphics[height=0.95\textheight, keepaspectratio]{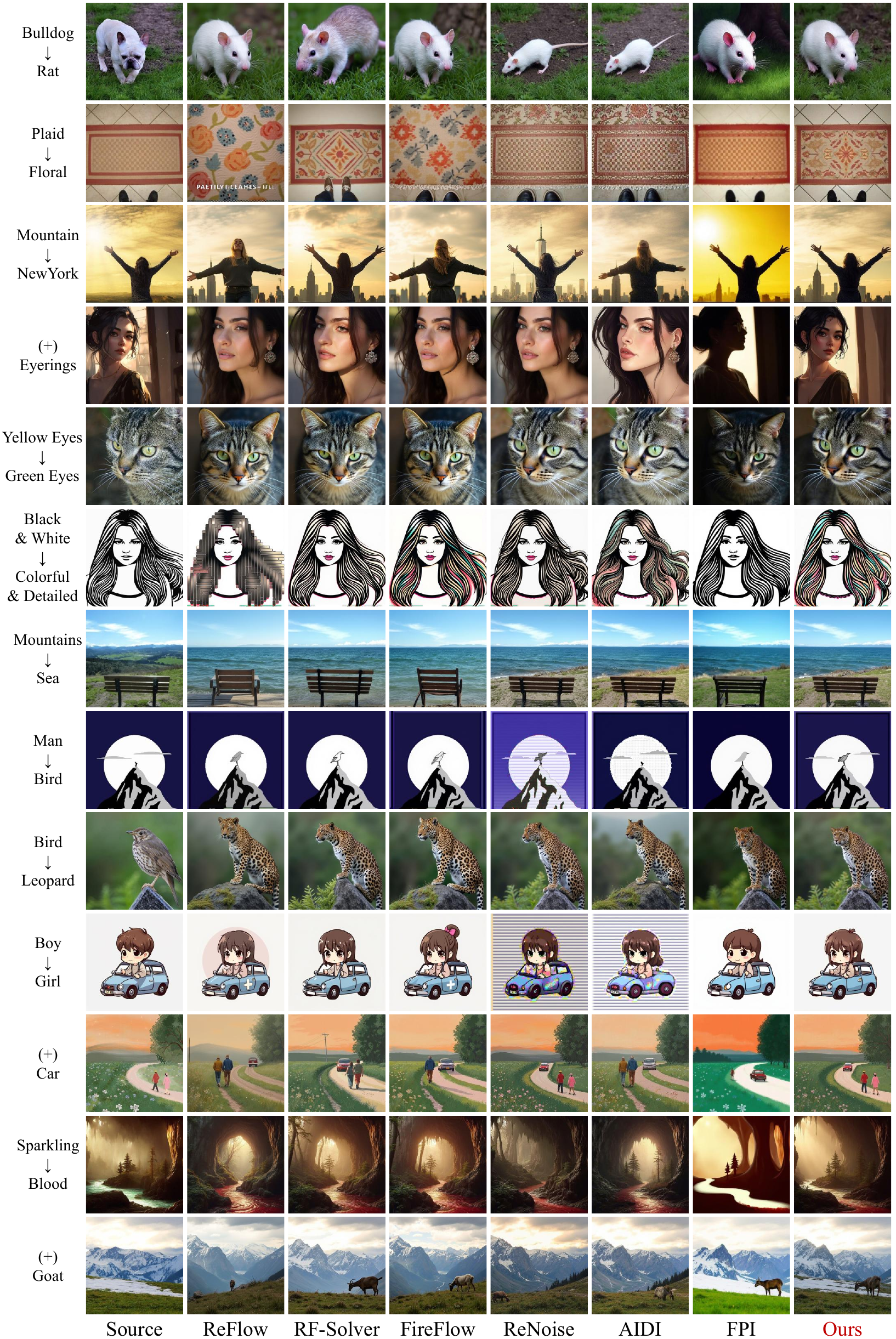}
    \caption{Additional qualitative results for the prompt-based image editing task.}
    \label{fig:app-additional-editing-grid}
\end{figure}

\end{document}